%% file: main.tex
\documentclass{article}

\usepackage{ijcai26}

\usepackage{booktabs}
\usepackage{times}
\usepackage{soul}
\usepackage{url}
\usepackage[hidelinks]{hyperref}
\usepackage[utf8]{inputenc}
\usepackage[small]{caption}
\usepackage{graphicx, subcaption}
\usepackage{amsmath}
\usepackage{amsthm}
\usepackage{natbib}
\usepackage{amssymb}
\usepackage{xcolor}
\usepackage{xspace}

\theoremstyle{plain}
\newtheorem{lemma}{Lemma}[section]

\title{Constraint-Data-Value-Maximization: Utilizing Data Attribution for Effective Data Pruning in Low-Data Environments\thanks{This is the extended arXiv version of the IJCAI 2026 paper. It includes proofs, ablations, and additional implementation details not present in the conference version.}}

\author{%
  Danilo Brajovic$^{1,2}$ \and David A. Kreplin$^{3}$ \And Marco F. Huber$^{1,2}$
  \affiliations
  $^{1}$ Fraunhofer IPA, Stuttgart, Germany \\
  $^{2}$ Institute of Industrial Manufacturing and Engineering IFF, University of Stuttgart, Germany \\
  $^{3}$ Hochschule Heilbronn, Germany
  \emails
  danilo.brajovic@ipa.fraunhofer.de, david.kreplin@hs-heilbronn.de, marco.huber@ieee.org
}

\def\cdvm{CDVM\xspace}

\begin{document}

\maketitle

\begin{abstract}
Attributing model behavior to training data is an evolving research field. 
A common benchmark is data removal, which involves eliminating data instances with either low or high values, then assessing a model's performance trained on the modified dataset.  Many existing studies leverage Shapley-based data values for this task. In this paper, we demonstrate that these data values are not optimally suited for pruning low-value data when only a limited amount of data remains. To address this limitation, we introduce the Constraint-Data-Value-Maximization (CDVM) approach, which effectively utilizes data attributions for pruning in low-data scenarios. By casting pruning as a constrained optimization that both maximizes total influence and penalizes excessive per-test contributions, CDVM delivers robust performance  when only a small fraction of the data is retained. On the OpenDataVal benchmark, CDVM shows strong performance and competitive runtime. 

\end{abstract}

\section{Introduction}
Machine learning models, especially large language models, have an insatiable demand for data, while the availability of data is stagnating. By attributing the influence of training data on model performance, the required amount of data can be reduced, thereby saving energy and improving model quality. Early works in this direction, such as influence functions \citep{Koh2017}, aim to gain insights into model behavior by attributing the influence of individual training instances on test instances, thereby serving as a method for explainable AI. Conversely, methods like data Shapley \citep{Ghorbani2019} have been used to assess the influence of single training instances on model performance, referring to this as data value, and applying this understanding for data removal. Typically, these approaches rely on Shapley-based  approximations to compute the value of each data instance.

\citet{Sorscher2022} benchmark methods for pruning data on ImageNet, showing that novel algorithms for data pruning can improve scaling laws, thus reducing the resource costs associated with modern deep learning.

The motivation for our work is that current data-pruning methods based on
Shapley or Banzhaf values suffer from inherent limitations that prevent them
from fully exploiting pruning potential. For Data Shapley in particular, this is aligned with \citet{Wang2024RethinkingDataShapleyForDataSelection}, who show that it need not outperform random selection unless the utility function satisfies structural assumptions. We first analyze these shortcomings
and then leverage our insights to design a new pruning algorithm. Our method formulates pruning as an optimization over a data attribution matrix and is evaluated on the OpenDataVal benchmark \citep{Jiang2023}, achieving strong performance while remaining computationally competitive against state-of-the-art techniques including those identified by \citet{Sorscher2022}.
Our findings show that there is still substantial room to improve data pruning, which in turn can lower training costs and reduce energy consumption.
Our main contributions are:

\begin{itemize}
  \item  We demonstrate that Shapley and Banzhaf-based attributions allocate smaller marginal contributions to instances in larger data clusters. This imbalance causes large clusters to be pruned too early, producing unbalanced removal patterns and suboptimal pruning performance.
  \item We show that optimal retention sets are budget-specific: the subset that maximizes accuracy for one retention budget need not fully contain the subset optimized for another budget.
  \item Based on these insights, we introduce Constraint-Data-Value-Maximization (CDVM), a novel algorithm that treats data pruning as an optimization problem over a data attribution matrix. 
  \item We benchmark CDVM on six datasets from OpenDataVal, showing strong performance. 
\end{itemize}

\section{Background, Motivation, and Related Work}
\label{sec:background}
We begin with a concise overview of data valuation, then examine pruning and other evaluation benchmarks, outline their limitations, and finally introduce two concepts that motivate our method.

\subsection{Data Valuation and Relationship to Pruning}
Data valuation assesses the overall impact of individual training instances on the model performance, effectively answering the question, "How much did a training instance contribute to the model's performance?"
The value assigned to each training instance \( i \) is represented as a scalar. Consequently, the valuation scores for a dataset are expressed as a vector \( v \in \mathbb{R}^{n} \), where $n$  is the number of training instances.  

\subsubsection{Estimating Data Values}
We now introduce the basic notation and the main estimation methods for data values used in this paper. For a comprehensive survey, see \citet{Hammoudeh2022} and \citet{Sim2022}.

\begin{itemize}
    \item $D=\{(x_i,y_i)\}_{i=1}^n$ is a labeled dataset with inputs $x_i$, labels $y_i$ and $d_i=(x_i,y_i)$.
    \item $f_D$ is the model trained on the dataset $D$. 
    \item $\theta_D$ are the corresponding model parameters.
    \item $f_{D \cup \{d_j\}}$ denotes a model trained on the union of $D$ and the  data instance $d_j=(x_j,y_j)$.
    \item $\mathcal{U}$ represents a utility function, such as accuracy in a classification setting.
\end{itemize}

\paragraph{Leave-One-Out}
The simplest approach to estimating the influence of a training instance is the leave-one-out (LOO) method, which involves excluding a particular data instance during training and comparing the model performance or test predictions with and without this instance. This method can be approximated by influence functions \citep{Koh2017} without the need for re-training. The main limitation is that the effect of omitting a single data instance can often be obscured by the remaining data and the inherent noise in the training process \citep{K2021}. As a result, many data instances may appear to have a negligible value. Empirical evidence also suggests that LOO is not effective for benchmarks in data valuation \citep{Jiang2023}.
Formally, the LOO value of data instance $d_i$ can be expressed as  
$V(d_i) = 
                \mathcal{U}(f_D) - \mathcal{U}(f_{D \setminus \{d_i\}}) 
$.

\paragraph{Semi-value-based Estimates}
\label{sec:shap_introduction}
Semi-value-based techniques quantify the importance of a training instance \(d_i\) by its marginal contribution over all subsets 
\(S \subseteq D \setminus \{d_i\}\). For data valuation, three variations were proposed; original Shapley value \citep{Ghorbani2019}, Banzhaf \citep{Wang2022}, and Beta Shapley \citep{Kwon2021}. Technically, all these methods differ only by the weighting of each subset $w(S)$ and can be expressed as
$
V(d_i) = 
    \sum _{S \subseteq D \setminus \{d_i\}}
        w(S)
            \left[ 
                \mathcal{U}(f_S) - \mathcal{U}(f_{S \cup \{d_i\}}) 
            \right]. 
$

These methods generally outperform the LOO estimate in practice. However, their main drawback is their exponential computational complexity. To mitigate this, Monte Carlo or other sampling-based techniques are often used to approximate data values. 
Notably, data Banzhaf \citep{Wang2022} has proved to be computationally efficient due to the \emph{Maximum Sample Reuse (MSR)} principle. 
The data value is approximated by sampling subsets \( S \subset D \) of the training data with probability $p$ and training a model on each subset. This process is repeated multiple times, and the data value of data instance $d_i$ is computed as the performance difference between subsets where $d_i$ is included versus where it is not. 

\paragraph{Out-of-Bag and Memorization Estimates}
The concept of memorization has been introduced in recent studies, wherein a training instance \(i\) is considered "rare" if its exclusion from the training set significantly reduces the probability that \(i\) is correctly classified by the same model \citep{Feldman2020MemorizationLongTail, Paul}. A related method used in data valuation is the out-of-bag estimate, known as \emph{DataOob}, where the significance of training instances is assessed using out-of-bag samples \citep{Kwon2023DataOob}. In each iteration, the training set is split into in-bag and out-of-bag groups, a model is trained on the in-bag samples, and predictions are made on the out-of-bag samples. The value of a data instance is then determined based on its memorization score during these out-of-bag assessments. 
Although these techniques are not suited for data attribution (as no test set is involved), they have proven effective in data pruning tasks, even on ImageNet \citep{Sorscher2022}.

\subsubsection{Data Pruning and other Benchmarks for Data Valuation}
Data-valuation methods are commonly evaluated on three tasks:

\begin{enumerate}
  \item \textbf{Noise Detection:} Identify and remove corrupted or mislabeled examples, which tend to carry large negative value due to their disruptive effect on training \citep{Jiang2023}.
  \item \textbf{Domain Transfer:} Select a subset of source--domain data that maximizes accuracy on a target--domain test set (e.g.,\ choosing MNIST digits to improve performance on street-number datasets) \citep{Ghorbani2019}.
  \item \textbf{Data Removal:} Measure how model accuracy changes when portions of the training set are removed in order of increasing or decreasing value. Removing high-value instances first should cause a steep accuracy drop, whereas pruning low-value instances should have minimal impact.
\end{enumerate}

In this work, we focus on data removal, since it directly addresses the practical goal of reducing dataset size without sacrificing performance. This is particularly relevant when models must be trained repeatedly under compute or storage constraints.
From here on, we use \emph{data pruning} to mean the removal of low-value points.
In the literature, authors differ in whether they report results for removing low-value data (pruning) or for removing high-value data first. For instance, the original Data Shapley study \citep{Ghorbani2019} presents low-value pruning curves, while the OpenDataVal framework \citep{Jiang2023} emphasizes high-value removal. We are not aware of any formal discussion explaining this discrepancy. Empirically, memorization-based or out-of-bag-based methods tend to excel at low-value pruning, whereas Shapley-based techniques often show stronger effects when high-value data is removed first.

\subsubsection{Limitations of Data Values for Data Pruning}
\label{sec:limitations}

\begin{figure*}[t]
    \centering
    \includegraphics[width=\linewidth]{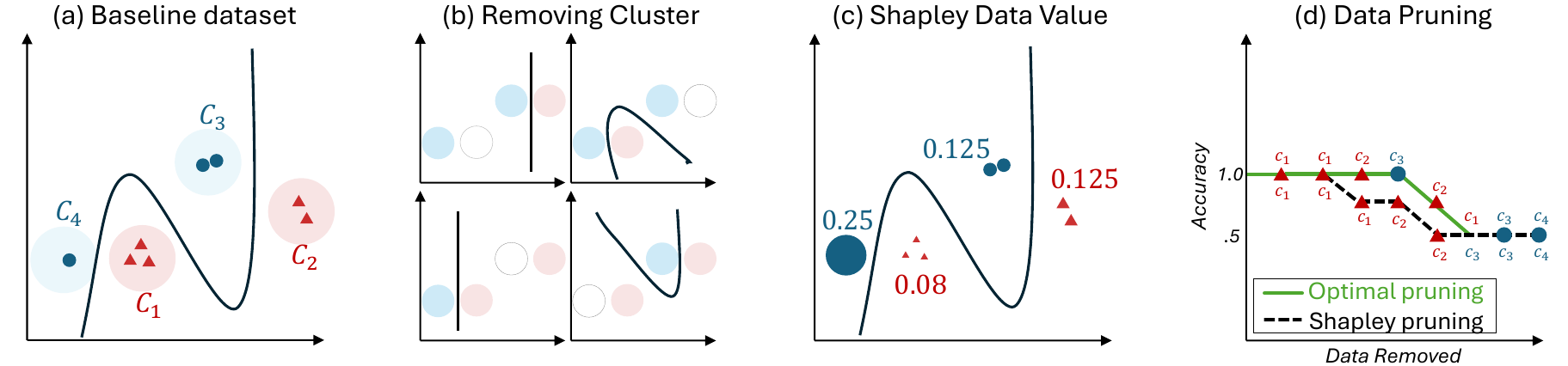}

    \caption{(a) Baseline synthetic dataset comprising 8 points from 4 clusters. 
    (b) Illustrates the changed decision boundary after removing an entire cluster. In each scenario, the decision boundary undergoes significant alterations. 
    (c) Displays the Shapley data value.
    (d) Test accuracy as we iteratively remove instances (x-axis: removal step 1--8; y-axis: accuracy).  
    In the optimal (green) removal order, clusters are pruned as 
    $c_1,c_1,c_2,c_3,c_2,c_1,c_3,c_4$, 
    whereas the Shapley-based (black) order is 
    $c_1,c_1,c_1,c_2,c_2,c_1,c_3,c_4$ and $c_i$ belongs to the $i$-th cluster.}
    \label{fig:synthetic_data}
\end{figure*}

After reviewing the main approaches to data valuation, we now highlight
their shortcomings in the context of data pruning, using a simple clustered
setting that we also analyze theoretically in
Appendix~\ref{sec:app_cluster_bias}.
Consider the dataset in Figure~\ref{fig:synthetic_data}(a): the training set
is partitioned into four Gaussian clusters
\[
  D \;=\; C_1 \cup C_2 \cup C_3 \cup C_4,
  \qquad C_r \cap C_s = \emptyset \;\text{for } r \neq s,
\]
where points in cluster \(C_k\) are sampled from a Gaussian distribution
\(\mathcal{N}(\mu_k, \sigma^2 I)\). Let \(n_k = |C_k|\) denote the size of
cluster \(C_k\); in this example, \(n_1 = 3\), \(n_2 = 2\), \(n_3 = 2\), and
\(n_4 = 1\). The black line shows the decision boundary learned by a
multi-layer perceptron trained on the full dataset.

The test set is partitioned analogously as
\[
\mathcal{T} = \mathcal{T}_1 \cup \mathcal{T}_2 \cup \mathcal{T}_3 \cup \mathcal{T}_4,
\quad |\mathcal{T}_k| = m_k
\]
so that \(\mathcal{T}_k\) contains the test points associated with cluster \(C_k\) (these
test points are not shown in the figure).

In this dataset, removing an entire training cluster \(C_k\) causes all
corresponding points in the test cluster \(T_k\) to be misclassified
(Figure~\ref{fig:synthetic_data}(b)). Hence each cluster \(C_k\) contributes a
utility \(u_k > 0\) if and only if at least one of its points is present in
the training subset.
Appendix~\ref{appendix:synthetic_data} provides further
details on this construction.

\paragraph{1. LOO has Redundancy Bias and Attributes Non-zero Value Only to Unique Data}
We begin with the observation that LOO attributions reward only non-redundant samples. In Figure \ref{fig:synthetic_data} (a), only the singleton instance in \(C_4\) receives a nonzero value $u_4>0$, since its removal introduces a change in the decision boundary (Figure \ref{fig:synthetic_data} (b)), causing a test error at the respective cluster center. All other instances receive a value of zero due to their redundancy and can therefore be pruned in any order.

\paragraph{2. Shapley and Banzhaf Exhibit Cluster-Size Bias and Cause Imbalanced Pruning}
In this model, Data Shapley and Data Banzhaf values share a cluster's total
utility among all its members: a point in $C_k$ receives Shapley value
$u_k / n_k$ (Figure \ref{fig:synthetic_data} (c), Appendix~\ref{sec:app_bias_shapley}) and Banzhaf value
$u_k / 2^{n_k-1}$ (Appendix~\ref{sec:app_bias_banzhaf}).

If all clusters share the same utility $u_k = u$ (e.g., when test clusters
$T_k$ are of equal size or when performance is measured via balanced accuracy
over clusters), then both Data Shapley and Data Banzhaf assign smaller values
to points in larger clusters (Appendix~\ref{subsec:app_interpretation_and_equal_train_test}),
so these points are pruned first. On the synthetic dataset, this leads to
initially harmless removal of redundant points from majority clusters, but
once the last representative of a cluster is removed, performance drops
sharply, in contrast to the optimal
strategy that keeps one representative per cluster for as long as possible (Figure~\ref{fig:synthetic_data}(d)).
Moreover, in this clustered model we can show that, when train and test sets are
equally distributed over clusters (i.e., the test set contains $m_k$ instances
from cluster $C_k$ with $m_k \propto n_k$), Shapley data values collapse to a
constant: every training instance receives the same Shapley value, regardless
of its cluster membership; see
Appendix~\ref{subsec:app_interpretation_and_equal_train_test} for the formal
derivation.

This effect can also be observed on real data. In the left plot of Figure \ref{fig:empirical_examples}, we compare DataBanzhaf against random pruning on CIFAR-10, using either 1,000 or 10,000 models to estimate data values. Both DataBanzhaf variants outperform random removal up to about 50\% pruning. Beyond that point, the 10,000-model variant plunges significantly below the random baseline, whereas the 1,000-model variant continues to slightly outperform random pruning, even though the larger ensemble should yield more accurate attributions. 

\paragraph{3. Pruning Subsets are Budget Specific}

Finally, we observe that optimal retention sets at different pruning levels are not nested: the subset that maximizes accuracy for one budget \(s\) may exclude instances that are essential for another budget \(s'\neq s\). In Figure~\ref{fig:empirical_examples} (center), we use our own method to identify the best subsets for retaining 5\%, 10\% and 15\% of the data.  We then perform sequential pruning of the remaining instances, always keeping the preselected subset intact and plot test accuracy versus fraction removed.  Each accuracy curve peaks exactly at its target retention level (dots), and even a slight deviation from that budget causes a dramatic collapse in performance. 
A similar pattern appears for Memorization/DataOob (Figure \ref{fig:empirical_examples}, right): removing the highest-value instances first (solid blue curve) initially improves performance before it plummets, whereas retaining those same instances until the very end yields almost state-of-the-art final accuracy. This mirrors the finding of \citet{Sorscher2022}, namely that the examples most dispensable in data-rich regimes are precisely those that must be kept when data become scarce.

In summary, these observations highlight the need for a pruning strategy that  
    (i) tracks influence at the level of individual test samples and  
    (ii) can flexibly re-optimize for each pruning budget.  
  To that end, we now review two key building blocks of our approach:  
  data attribution and influence-function--guided pruning.

\begin{figure*}[t]
    \centering    
    \includegraphics[width=.32\linewidth]{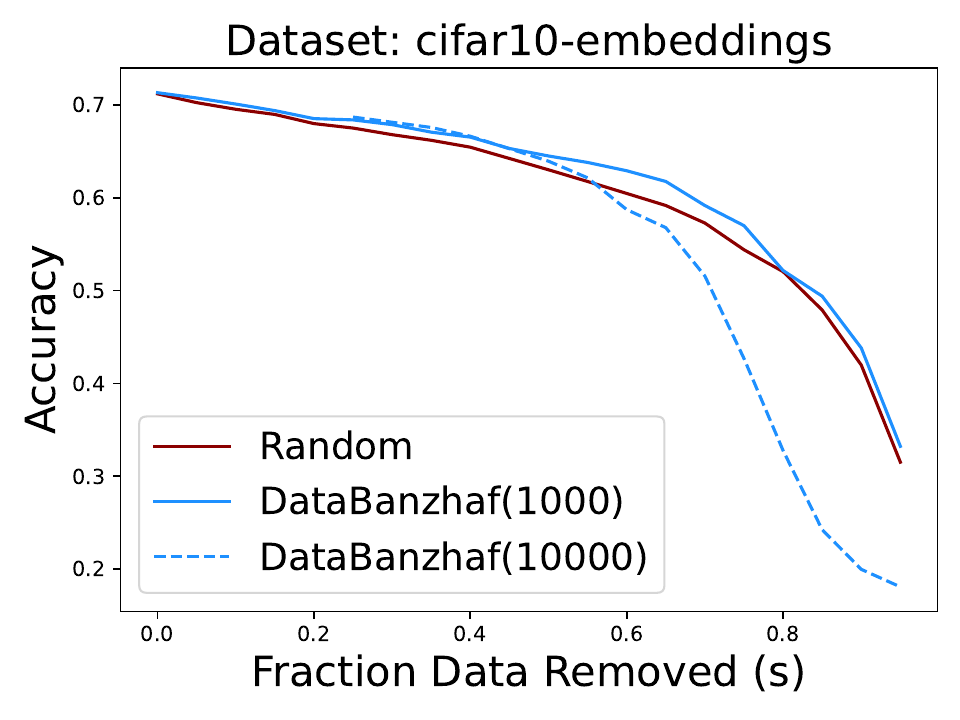}
    \includegraphics[width=0.32\linewidth]{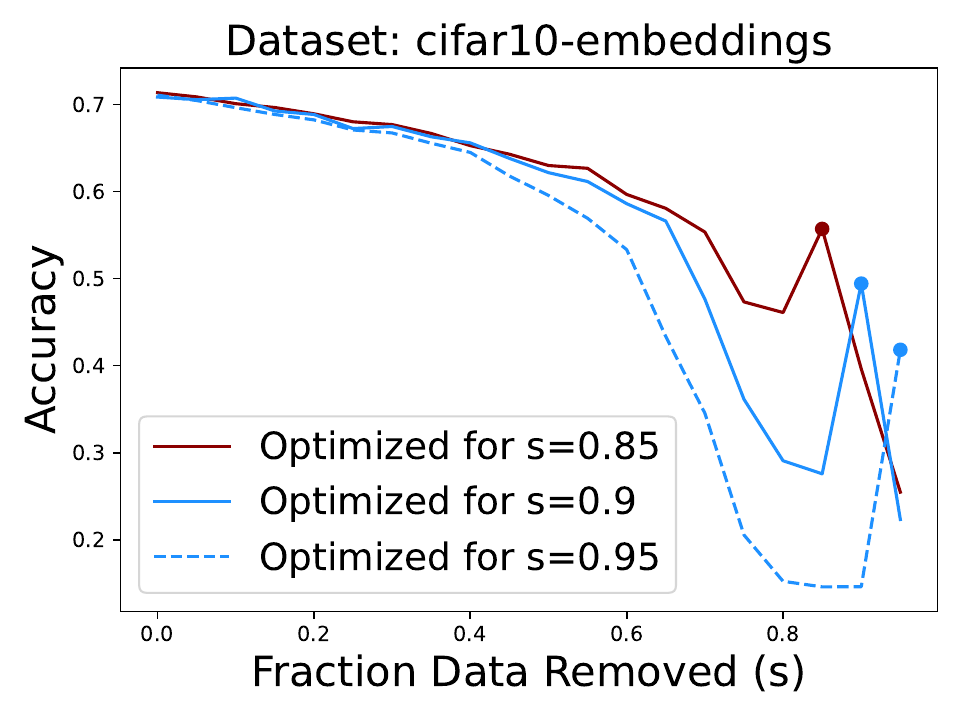}
    \includegraphics[width=.32\linewidth]{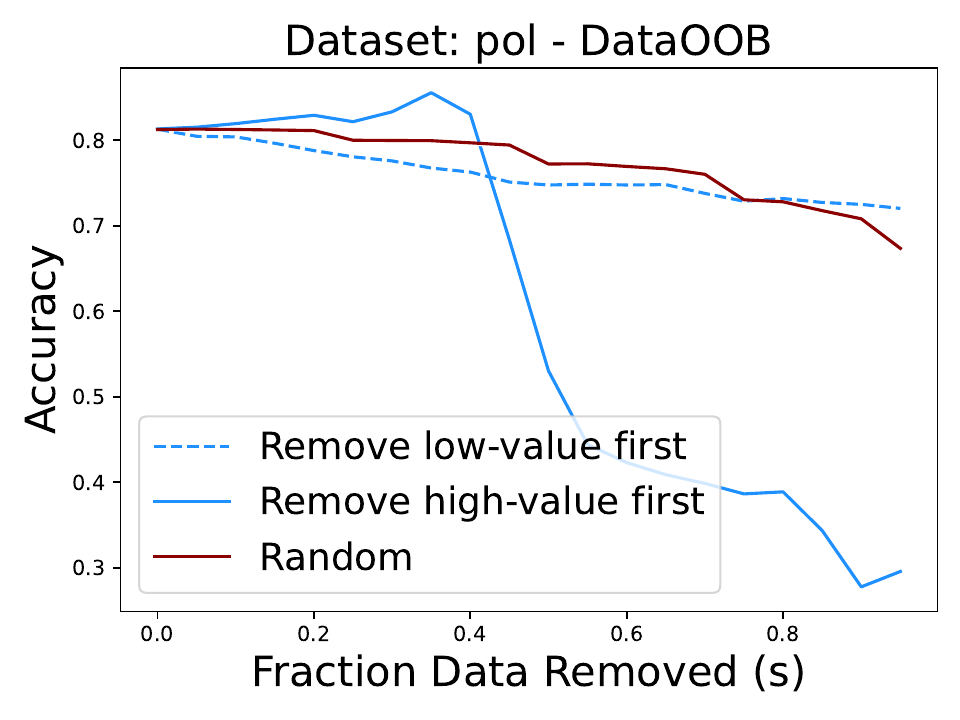}

    \caption{
    All plots show test accuracy as a function of the fraction of training data removed.
    \textbf{Left:} CIFAR-10 results for DataBanzhaf pruning. We estimate Banzhaf values with 1,000 models (solid blue) and 10,000 models (dashed blue), and compare against random removal (solid red). Both Banzhaf variants outperform random up to ~50\% pruning, but counterintuitively the 10,000-model variant degrades faster than the 1,000-model version.
    \textbf{Center:} An extreme example of non-nested pruning subsets. Each curve is optimized for exactly 85\%, 90\%, or 95\% removal (i.e., 15\%, 10\%, 5\% retention). Accuracy peaks precisely at the target rate (dots), and removing more or less data causes a steep collapse.
    \textbf{Right:} Memorization/DataOob pruning. The dashed blue curve removes lowest-value instances first; the solid blue curve removes highest-value instances first; and the red curve is random removal. Surprisingly, the very instances whose early removal boosts accuracy (and outperforms random) when data is abundant, must be retained until the end under high removal budgets to again outperform the random baseline.
    }
    
    \label{fig:empirical_examples}
    
\end{figure*}

\subsection{Other Related Work}
Core-set methods select small subsets that preserve the geometry or distribution
of the full dataset \citep{Feldman2020CoreSets, Jiang2025LearningComplexityDownstreamPruning}.
Recent work combines importance scores with redundancy penalties in discrete
optimization formulations \citep{Tan2025PruningByInfoMax}, conceptually related
to our CDVM objective. In parallel, Most Influential Subset Selection (MISS)
studies subset-level influence and its failure modes \citep{HU2025_MISS}, and
several pruning methods leverage training dynamics such as example difficulty
or uncertainty \citep{Cho2025_LigweightPruningWithDifficultyAndUncertainty,
He2024_PrunungWithUncertainty}. We view CDVM as complementary to these lines,
focusing on attribution-based pruning in low-data regimes.

\subsection{Preliminaries: Data Attribution \& Influence-Function Pruning}
\label{sec:preliminaries}
We now introduce two core concepts underlying our method: data attribution and influence-function--guided pruning.

\subsubsection{Data Attribution}
Data attribution is conceptually related to data valuation, but traces the influence of individual train instances down to specific test samples. 
The influence of training data on test predictions is quantified using the attribution matrix \( \mathbf{T} \in \mathbb{R}^{n \times m} \), where \( n \) is the number of train instances and \( m \)  the number of test instances. A high value of \( \mathbf{T}_{i,j} \) indicates that the train instance \( i \) significantly impacts the prediction for test instance \( j \). The connection between both is that data values  can be estimated by averaging over the columns (test instances) of \( \mathbf{T}\), formulated as \( v_i = \frac{1}{m} \sum_{k=1}^m \mathbf{T}_{i,k} \). This per--test--sample breakdown provides a fine-grained view of dataset contributions, which we leverage directly in our method.
Several methodologies have been developed to estimate this influence, with influence functions being one of the pioneering approaches \citep{Koh2017}. TRAK has emerged as a scalable method for data attribution across large datasets  \citep{Park2023}. 

\subsubsection{Influence-Function--Guided Pruning}
 \citet{Yang2022_DataSetPruningInfluenceFunctions} cast data pruning as a discrete optimization problem over binary removal variables, with the goal of minimizing overall parameter change. Let $r \in \{0,1\}^n$
be the indicator vector specifying which of the \(n\) training samples are removed. The parameter change from removing a single instance \(d_i\) is given by the influence function
\(
\mathcal{I}(d_i)
= \theta_{D\setminus d_i} - \theta_D
\approx \frac{1}{n}\,H_{\theta}^{-1}\,\nabla_{\theta}\mathcal{L}(d_i; \theta_D)\,,
\)
where \(H_{\theta}\) is the Hessian of the total training loss at \(\theta_D\). For a subset of removed instances, these influences simply add up.  Define the matrix
\(
\mathbf{Z} = \bigl[\mathcal{I}(d_1),\,\dots,\,\mathcal{I}(d_n)\bigr],
\)
so that the total parameter change caused by the selected removals is \(\mathbf{Z}\,r\). They then solve
\[
\min_{r\in\{0,1\}^n}\;\bigl\|\mathbf{Z}\,r\bigr\|_2
\quad\text{s.t.}\quad
\sum_{i=1}^n r_i = B\,,
\]
where \(B\) denotes the removal budget. Although this method achieves strong empirical performance and inspired our approach, it has two major drawbacks. First, it requires (approximate) Hessian inversion for every training instance, which is computationally expensive. 
Second, because it relies on influence functions, essentially approximating leave-one-out, it inherits LOO's limitations (see Sec.~\ref{sec:limitations}).

\section{Size-Constraint Data-Value-Maximization}
\label{sec:cdvm}
Building on the limitations identified in Section~\ref{sec:limitations}, we now introduce a method that derives data values specifically optimized for pruning. There, we observed that Shapley-based data values can fail because they tend to remove entire clusters first. To overcome this, we leverage the attribution matrix
\[
\mathbf{T} \in \mathbb{R}^{n \times m},
\]
which describes the influence of each of the \(n\) training samples on each of the \(m\) test samples. A naive way to derive pruning scores from \(\mathbf{T}\) is to average over its columns, but this approach suffers (among other issues) from the cluster-removal limitation noted above. Instead, \(\mathbf{T}\) provides fine-grained, per-test influence values that do not suffer from redundancy bias. We leverage this to ensure balanced coverage: at each pruning step, no test sample (and thus no implicit cluster) should have zero total influence. To formalize this, let
\[
w \in \{0,1\}^n
\]
be the binary indicator vector selecting exactly \(S\) out of the \(n\) training instances to retain. The induced utility vector for the \(m\) test samples is
\[
v = \mathbf{T}^\top w \;\in\;\mathbb{R}^m ~,\quad
v_j = \sum_{i=1}^n \mathbf{T}_{i j} \, w_i ~.
\]
A naive pruning objective would be
\[
\max_{w}\;\sum_{j=1}^m v_j
\quad\text{s.t.}\quad
\sum_{i=1}^n w_i = S,\quad
w_i\in\{0,1\}.
\]
This objective maximizes total influence but can still concentrate all value on a few test instances. To ensure balanced coverage, we introduce nonnegative slack variables \(t_j\) that caps any excess above a threshold \(\kappa\). In other words, any amount \(\max\{v_j - \kappa,\,0\}\) is transferred into \(t_j\) and subtracted from the objective. We call the resulting formulation Constraint Data-Value Maximization (CDVM):

\[
\begin{aligned}
    \max_{w, t}\quad &\alpha \sum_{j=1}^m v_j \;-\;(1-\alpha)\sum_{j=1}^m t_j~,\\
    \text{s.t.}\quad
    &v = \mathbf{T}^\top w~,\\
    &\sum_{i=1}^n w_i = S~,\\
    &t_j \ge 0~,\quad j=1,\dots,m~,\\
    &t_j \ge v_j - \kappa ~,\quad j=1,\dots,m~,\\
    &w_i \in \{0,1\}~,\quad i=1,\dots,n ~.
\end{aligned}
\]

This formulation directly remedies the shortcomings identified in Section~\ref{sec:limitations} by 
(1) maximizing each test sample's total influence via \(\mathbf{T}\), thereby avoiding redundancy bias; 
(2) penalizing any excess above \(\kappa\), thereby discouraging concentration of influence on a small number of test samples; 
and (3) enforcing a fixed subset size \(S\) to identify the optimal subset for the given budget.  
Furthermore, because all constraints are linear and some decision variables are integer-valued, the problem can be formulated as a mixed-integer linear program.

In the clustered setting from before, the analysis in Appendix~\ref{sec:app_cdvm_cluster} shows that, under the corresponding parameter setting, CDVM preserves at least one representative per cluster whenever \(S \ge K\).

\subsection{Implementation Details}
In our final setup, we relax \(w_i \in \{0,1\}\) to \(w_i \in [0,1]\) and retain the top-\(S\) entries of \(w\) after solving the linear program.
This improves tractability without observable loss; about 10\% of entries of \(w\) are fractional overall (roughly 6--23\% by dataset), with smaller fractions at smaller retention sizes.
The algorithm takes as input the attribution matrix \(\mathbf{T}\) and two hyperparameters:
\begin{itemize}
  \item \(\alpha \in [0,1]\): trade-off between total utility and penalty for exceeding \(\kappa\),
  \item \(\kappa\): soft upper bound on the influence per test sample.
\end{itemize}

Computing \(\mathbf{T}\) is the main computational bottleneck, since it requires retraining models on sampled subsets, an expense shared by all semi-value-based methods. Consequently, any parameter used to estimate $\mathbf{T}$ effectively becomes a hyperparameter. Here, we follow the Maximum Sample Reuse (MSR) principle of \citet{Wang2022}:  
\begin{enumerate}
  \item Sample \(T\) subsets \(S_t \subseteq D\) by including each training instance with probability \(p\).
  \item Train a model on each \(S_t\) and record the performance (or indicator of correct classification) on each validation instance.
  \item Estimate \(\mathbf{T}_{ij}\) as the average difference in that performance for validation instance \(j\) when \(d_i\) is in versus out of \(S_t\).
\end{enumerate}
The entries \(\mathbf{T}_{ij}\in[-1,1]\) are easily interpretable:  
\(-1\) means ``always causes a mistake'' and  
\(+1\) means ``always ensures correct prediction.'' Moreover, $\mathbf{T}$ is sparse, most training instances have zero or negligible influence on most validation instances, which significantly accelerates subsequent optimization.

Once \(\mathbf{T}\) is computed,  we solve the relaxed CDVM problem using the Disciplined Parametrized Programming framework.  This formulation enables caching, so we can quickly resolve the program after it has been solved once. 
This efficiency allows a lightweight grid search over the two hyperparameters \(\alpha\) and \(\kappa\). We then run the optimization independently for each retained fraction (e.g.,\ 30\%, 25\%).

\section{Experimental Results}
\label{sec:experiments}

\begin{figure*}[t]
    \centering
    \includegraphics[width=0.32\linewidth]{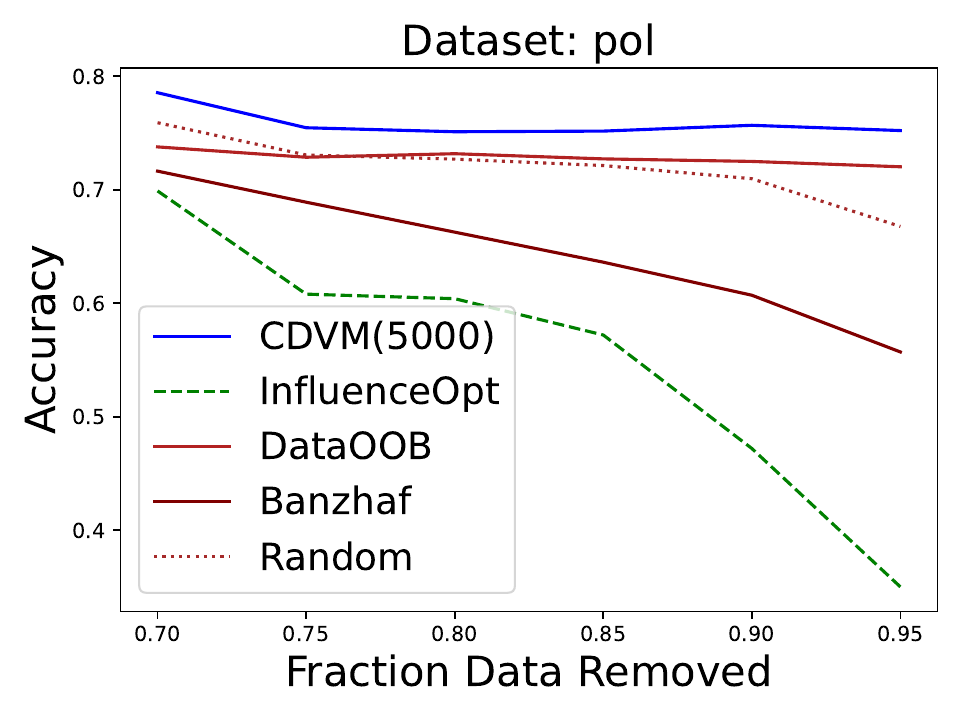}
    \includegraphics[width=0.32\linewidth]{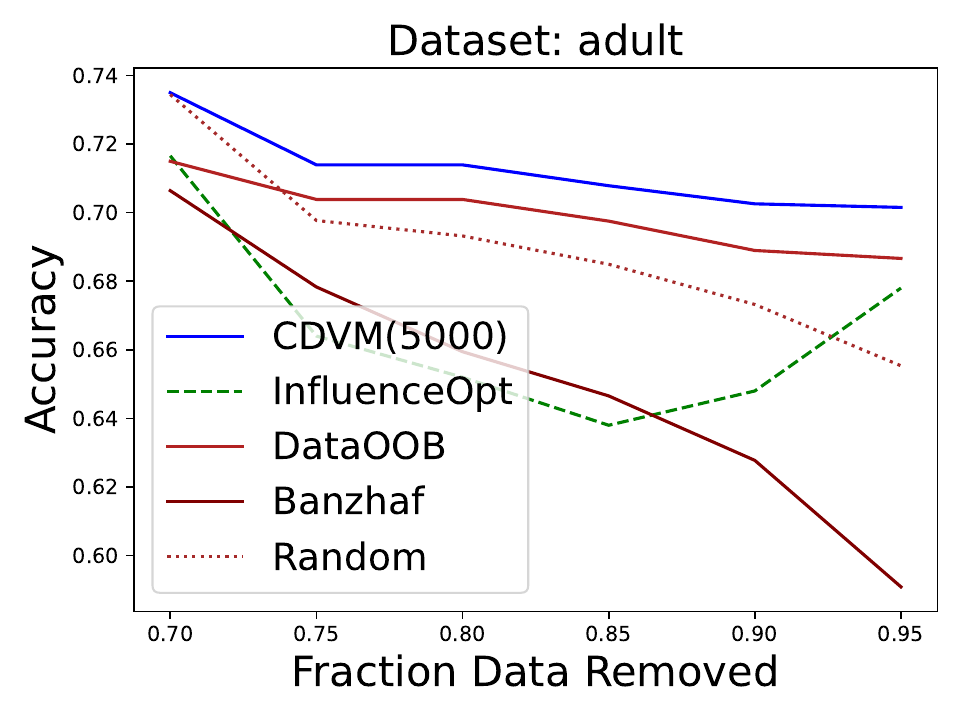}
    \includegraphics[width=0.32\linewidth]{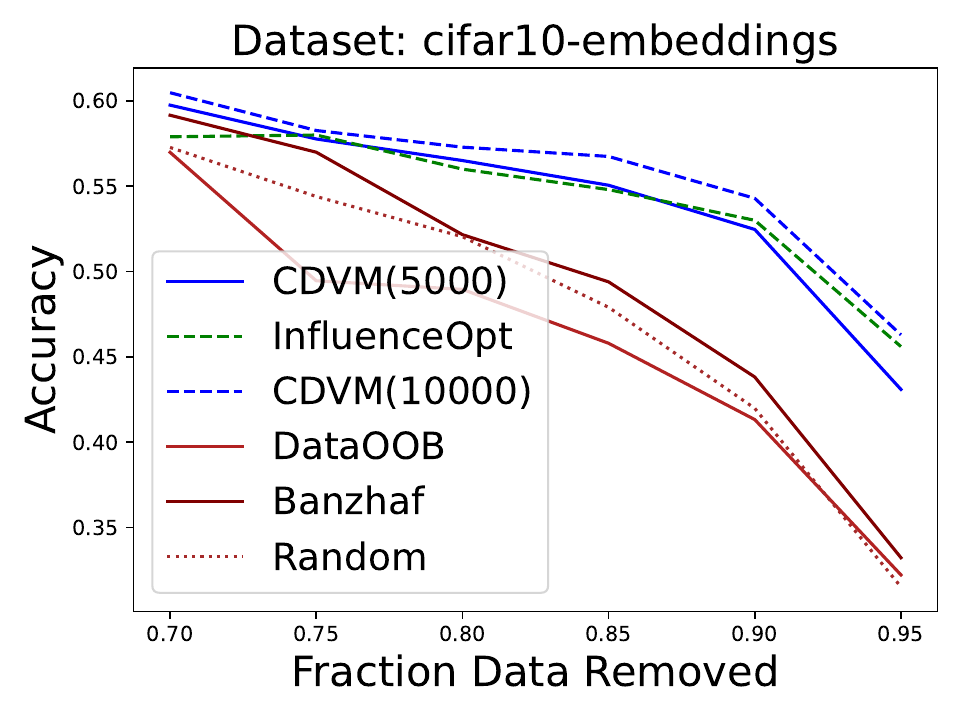}
    
    \includegraphics[width=0.32\linewidth]{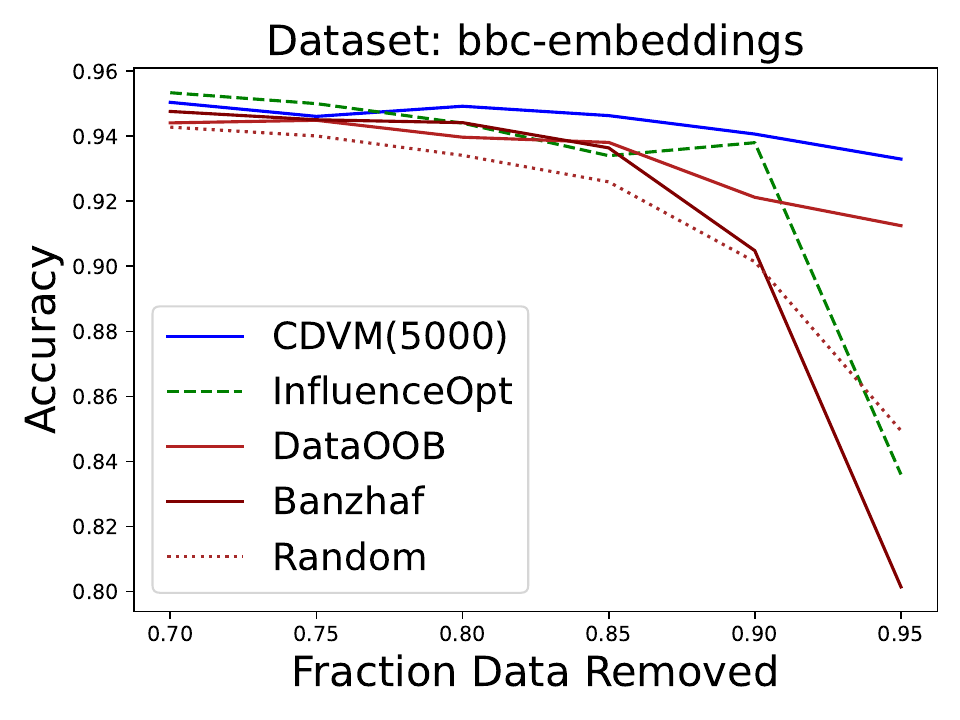}
    \includegraphics[width=0.32\linewidth]{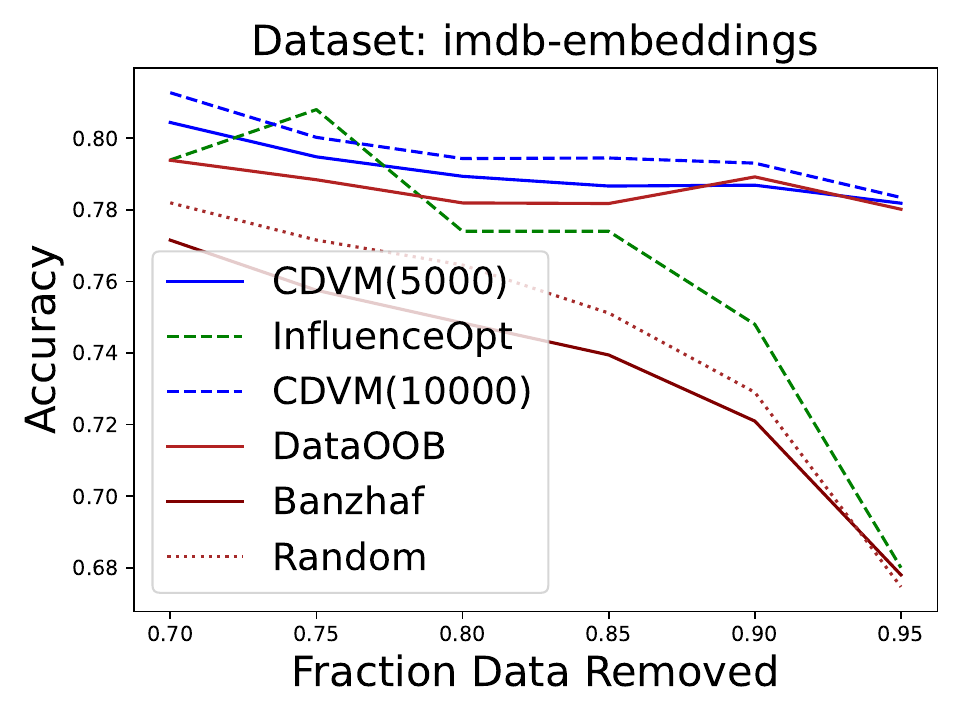}
    \includegraphics[width=0.32\linewidth]{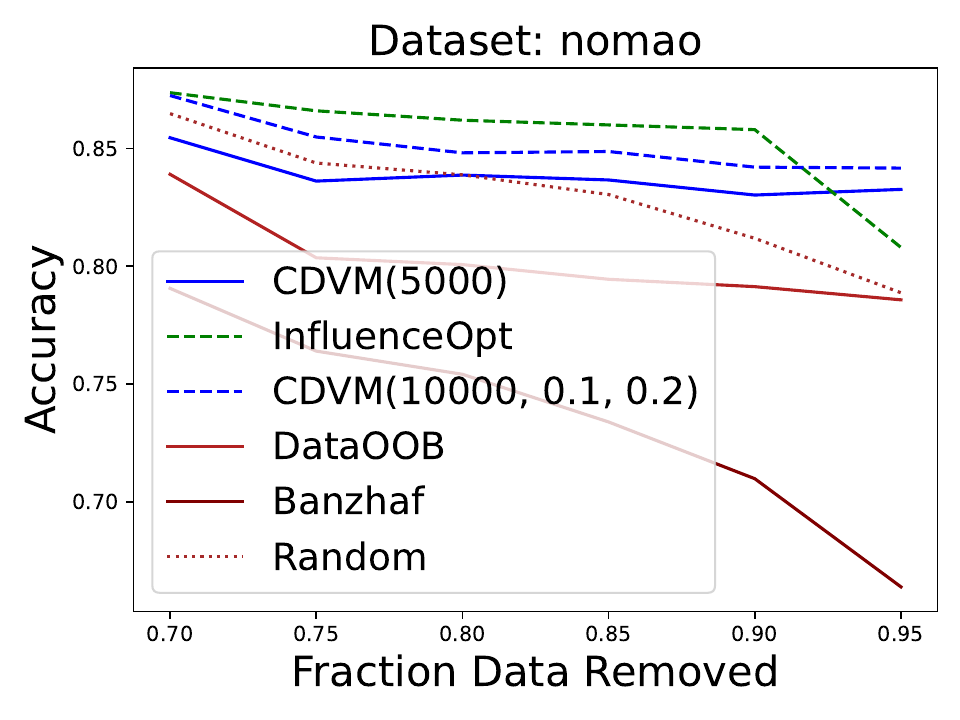}
    
    \caption{
    Accuracy on 30\%, 25\%, 20\%, 15\%, 10\%, and 5\% of remaining training data for six datasets in the OpenDataVal benchmark \citep{Jiang2023}. We utilized a sampling probability of \( p=0.03 \) for computing the attribution matrix and automatically optimized parameters for the \cdvm method. Out of 36 configurations, \cdvm achieved state-of-the-art performance in 28 setups. 
    }
    \label{fig:main_results}
\end{figure*}

We evaluate CDVM on the six datasets from the OpenDataVal benchmark \citep{Jiang2023}.  
The attribution matrix $\mathbf{T}$ is computed on the training--validation split and used to prune training instances. Final performance is then assessed on the held-out test set. 
Each experiment is run with 25 seeds, and we report the average test accuracy at retention levels from 5\% to 30\% in steps of 5\%.
We focus on low retention levels because differences are small at high retention and become pronounced only once important instances start to be pruned.
We compare against the following baselines:
\begin{itemize}
  \item \textbf{Random} removal of training samples.
  \item \textbf{DataOob/memorization} showing strong performance in \citet{Kwon2023DataOob, Sorscher2022}.
  \item \textbf{DataBanzhaf} \citep{Wang2022}, a semi-value--based method grounded in the MSR principle, which we also employ.
  \item \textbf{Influence Optimization} (\citet{Yang2022_DataSetPruningInfluenceFunctions}).  We encountered some stability issues with the original code: e.g.\ optimizing for a 10\% final subset occasionally performed better when using the budget for 5\%.  To avoid underestimating this baseline, at each pruning level we compare against the best accuracy achieved by this method over any budget. We also relaxed the constraint \(r_i\in\{0,1\}\) to a continuous one \(r_i\in[0,1]\), since the original mixed-integer-programming formulation often failed to converge. Consequently, these results should be viewed as an upper bound on the method's performance.
\end{itemize}

DataBanzhaf serves as a baseline to ensure any performance gains stem from our optimization rather than the attribution algorithm.
For CDVM, we fix the sampling probability at \(p=0.03\) and train \(T=5,000\) models (the primary computational bottleneck).  In some cases, tuning \(p\) or increasing \(T\) manually yields gains; we also report those as dashed line when they are significant.

Figure~\ref{fig:main_results} summarizes results over the 36 evaluation instances (6 datasets $\times$ 6 pruning rates), our default CDVM configuration (solid blue) outperforms baselines in 24 cases. Per-dataset tuning (dashed blue) yields some gains and increases the total number of state-of-the-art results to 28. Appendix \ref{appendix:results} provides the full tabular breakdown.

Among all baselines, only \citet{Yang2022_DataSetPruningInfluenceFunctions} (green dashed line) outperforms CDVM, mainly on the \texttt{nomao} dataset, where it edges out CDVM at 5 of the 6 pruning levels. It also performs competitively on \texttt{cifar10} but falls short on \texttt{pol} and \texttt{adult}, despite being evaluated as an upper bound. On the two text datasets (\texttt{bbc} and \texttt{imdb}), its gains are occasional and smaller. In contrast, CDVM is the only method that consistently beats random across all six benchmarks, provided hyperparameters are appropriately tuned (see nomao discussion below). DataOob/memorization remains competitive on \texttt{imdb} and \texttt{bbc} datasets, but never achieves state-of-the-art. 

The \texttt{nomao} dataset exhibits unusual dynamics for CDVM and \citet{Yang2022_DataSetPruningInfluenceFunctions}'s methods. With default hyperparameters, CDVM initially underperforms random pruning up to an 85\% removal rate. We found that manually tuning to \(p=0.1\), \(\kappa=0.2\), \(\alpha=0.1\) restores its advantage. Likewise, \citet{Yang2022_DataSetPruningInfluenceFunctions}'s approach attains its best scores on \texttt{nomao} only when its ranked instances are removed first and the remainder are kept at random, i.e., by applying its ranking in reverse.
We attribute CDVM's initial failure mode on this dataset to the high proportion of near-zero entries in the attribution matrix \(\mathbf{T}\). Increasing the sampling probability \(p\) seems to improve this, and using a smaller threshold \(\kappa\) prevents CDVM from assigning too much influence to individual validation instances when overall attributions are small.

\subsection{Ablation Study}
\label{sec:ablation}

\begin{figure*}[!ht]
    \centering
    \includegraphics[width=0.24\linewidth]{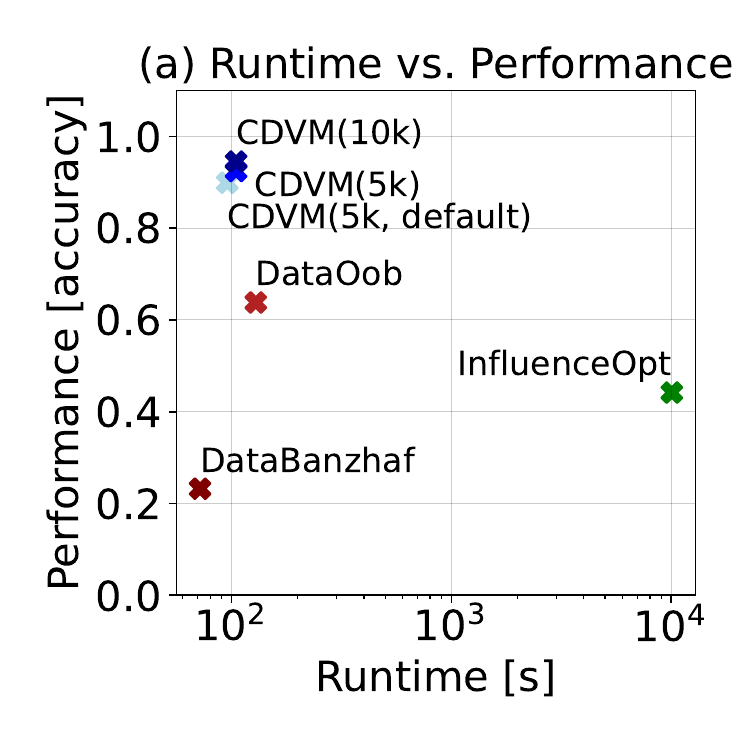}
    \includegraphics[width=0.24\linewidth]{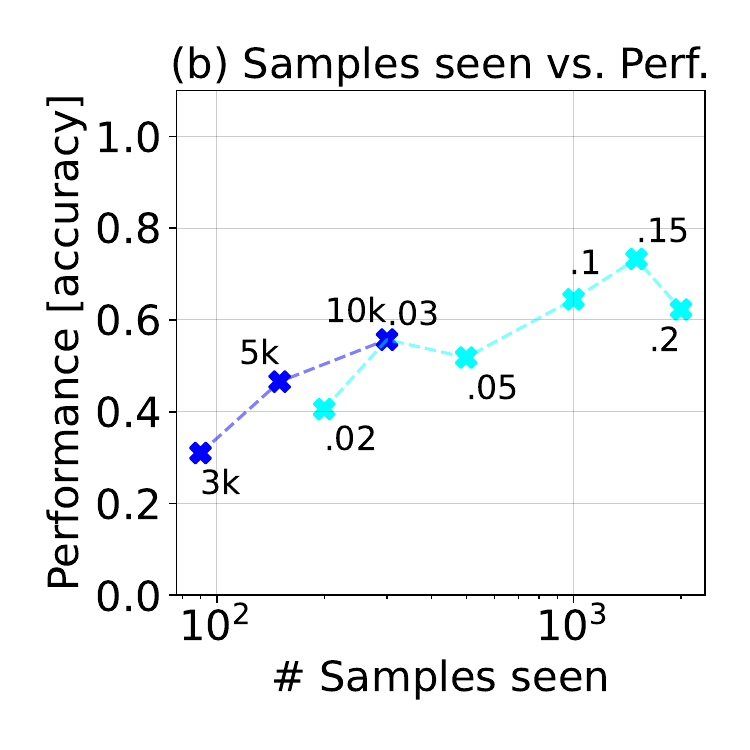}
    \includegraphics[width=0.24\linewidth]{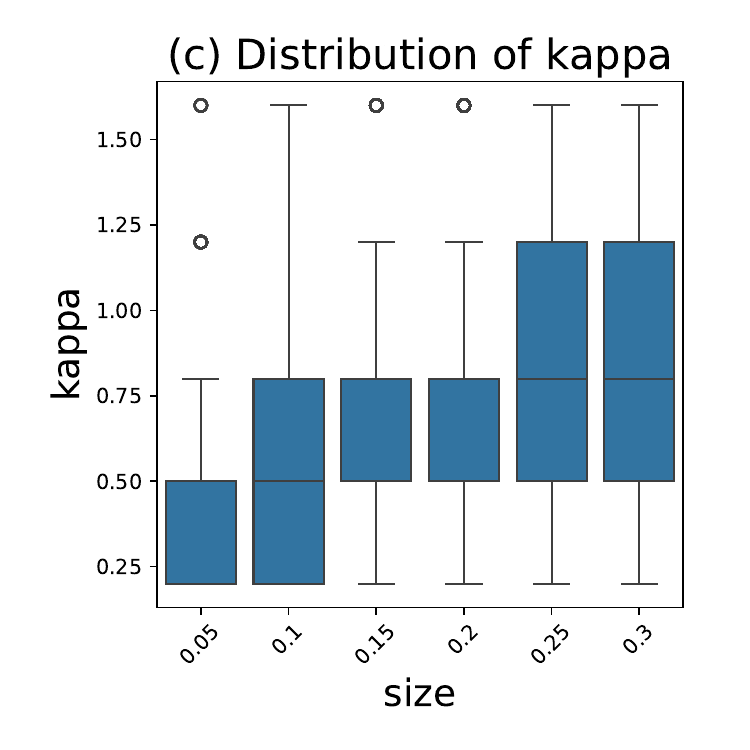}
    \includegraphics[width=0.24\linewidth]{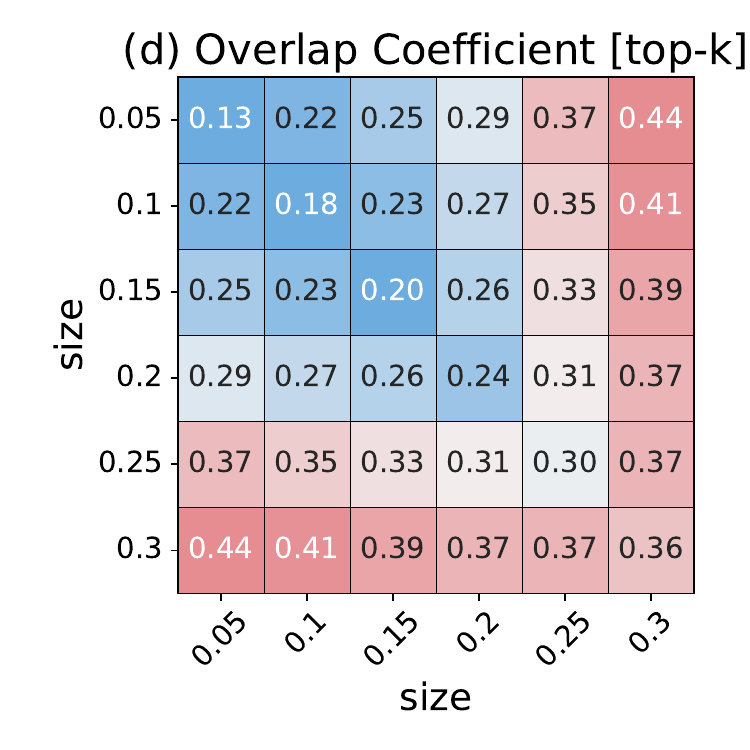}
    \caption{ 
    \textbf{(a)} Runtime vs.\ normalized performance for all benchmarked methods, aggregated over six datasets and six pruning levels.  
    \textbf{(b)} CDVM performance as a function of how often each sample is seen during training for sampling probability \(p\) {\color{cyan}(cyan)} and number of models trained  \(T\) {\color{blue}(blue)}.  
    \textbf{(c)} Distribution of the selected slack threshold \(\kappa\) across datasets and retention fractions.  
    \textbf{(d)} Average overlap coefficient $\frac{|A \cap B|}{\min(|A|, |B|)}$ between retained sets at different retention levels, aggregated across datasets and seed pairs. Diagonal values increase with retention size, reflecting greater redundancy at smaller budgets; off-diagonal values exceed diagonal values, indicating that larger retained subsets largely contain the smaller ones.
    }
    \label{fig:ablation}
\end{figure*}

\paragraph{Runtime Comparison}
Figure~\ref{fig:ablation}(a) plots each method's average runtime against its
normalized performance across all 36 settings, where 1 denotes the highest
and 0 the lowest accuracy across all methods and budgets (details in
Appendix~\ref{app:ablation_study}). CDVM achieves the best speed--accuracy
trade-off by a wide margin; DataOob outperforms Influence Optimization in
efficiency by delivering consistently strong accuracy at lower computational
cost. Figure~\ref{fig:ablation}(b) further shows that performance improves as
the number of times each sample is observed ($p \cdot T$) increases, but
increasing~$T$ at fixed~$p$ yields more stable gains than enlarging subset
sizes at fixed~$T$.

\paragraph{Hyperparameter Sensitivity}
Although CDVM introduces two hyperparameters (\(\alpha\) and \(\kappa\)), we find them robust across tasks and can be set without a grid search. We recommend
\[
  \alpha = 0.5,
  \quad
  \kappa = \max_{i,j}\mathbf{T}_{ij} \;+\; S\;\mathrm{mean}_{i,j}(\mathbf{T}_{ij}),
\]
which adapts \(\kappa\) automatically to the dataset and the retention budget \(S\).
The intuition behind this choice is that the effective budget per validation instance should scale with the size of the retained subset: on average, a training example contributes \(|S|\) times the mean attribution value. To ensure that highly influential instances are not overly penalized, we additionally add the maximum attribution value.
Figure~\ref{fig:ablation}(c) confirms this empirically: the $\kappa$ values
selected by grid search increase consistently with the retention size~$S$. 

In Figure~\ref{fig:ablation}(a) we compare CDVM(5k) with grid-searched hyperparameters against CDVM(5k, default) using the settings above. The performance difference across all datasets and budgets is marginal ($\le$ 0.05 in normalized performance) while the default configuration incurs zero search overhead, making it a practical choice for most applications.

\paragraph{Nestedness}
To better understand nestedness, we compute the overlap coefficient between retained sets. Figure~\ref{fig:ablation}(d) shows that diagonal values increase with retention size, indicating greater redundancy at smaller budgets. Off the diagonal, larger retained subsets largely contain the smaller ones, showing some nestedness across budgets. This suggests a three-fold structure: instances never selected at any budget, a pool of generally useful instances shared across budgets, and budget-specific instances whose value depends on the retention level. Further analysis is provided in Appendix~\ref{app:nestedness}.

\section{Scalability and Computational Cost}

By default, each dataset in OpenDataVal is subsampled into smaller splits.   
To demonstrate that CDVM extends beyond these reduced settings, we include in Appendix~\ref{app:scaling} a scaling experiment on the Fashion-MNIST dataset (60{,}000 train, 10{,}000 test).  
Thanks to the sparsity of the attribution matrix \(\mathbf{T}\), retaining only 5--10\% of its entries yields solve times of 10--30 minutes, including hyperparameter grid search, whereas retraining models to estimate \(\mathbf{T}\) (a cost shared by all semi-value-based approaches) requires several hours.
All experiments were conducted on a single consumer-grade workstation without specialized hardware.

Extending CDVM to even larger datasets such as ImageNet introduces two main bottlenecks: (i) estimating \(\mathbf T\) in terms of compute and memory, and (ii) solving the linear program at that scale. For context, prior studies have precomputed influence or memorization estimates on ImageNet by training thousands of ResNet-50 models\footnote{\url{https://pluskid.github.io/influence-memorization/}} demonstrating that such training is feasible even on ImageNet, although the resulting attribution matrix is very large (\(\approx 250\,\mathrm{GB}\) for train\(\times\)test). By thresholding \(\mathbf T\) to keep only its top 10\% nonzeros, this can be reduced to \(\approx25\,\mathrm{GB}\).

On the computational side, \(\mathbf{T}\) could be approximated using similarity measures between train and validation samples in feature space, in the spirit of \citet{Sorscher2022}. If the resulting linear program is still too large, column generation and warm starts across budgets could further reduce solve time. We leave these extensions to future work.

\section{Summary, Limitations \& Outlook}
In this work, we introduced Constraint-Data-Value-Maximization (CDVM), an optimization-based framework that leverages the data-attribution matrix $\mathbf{T}$ to prune low-value examples in low-data regimes. Across six OpenDataVal tasks, CDVM achieved strong performance while remaining computationally competitive.

The main bottleneck of CDVM, and of other approaches based on semi-values, is the computation and storage of \(\mathbf{T}\). On the computational side, future work could explore more efficient attribution estimators. On the storage side, exploiting sparsity or low-rank structure in $\mathbf{T}$ would substantially reduce memory usage and, in turn, speed up optimization. Additional gains in optimization time may be possible by solving the CDVM problem on partitioned submatrices of $\mathbf{T}$ rather than on the full matrix at once.

Finally, because the entries of $\mathbf{T}$ are not additive, CDVM does not explicitly capture higher-order interactions. Incorporating Shapley interaction indices \citep{muschalik2024shapiq} is an interesting direction to better model these effects.

\bibliographystyle{named}
\bibliography{library}

\section*{Use of Large Language Models (LLMs)}
LLMs were primarily used to enhance the paper's language and support code completion during implementation, as well as to define, refine, and improve the optimization problem. Although the initial concept originated with the authors, LLMs contributed  refinements and performance optimizations.

\section*{Reproducibility statement}
To ensure reproducibility, we provide the source code at \url{https://github.com/danilobr94/ijcai2026_cdvm}. The repository contains the code required to reproduce the experiments and plots reported in the paper via MLflow, as well as a standalone Jupyter notebook that computes the data-attribution matrix $\mathbf{T}$, formulates and solves the CDVM optimization, and compares the resulting subsets against a random baseline.

\appendix
\include{appendix}

\end{document}

%% file: appendix.tex
\clearpage
\onecolumn

\section{Notation Summary}
\label{app:notation}
Table~\ref{tab:notation_summary} summarizes the main notation used throughout the paper and appendix.

\begin{table}[!h]
  \centering
  \renewcommand{\arraystretch}{1.2}
  \begin{tabular}{p{0.22\linewidth}p{0.70\linewidth}}
  \toprule
  \textbf{Symbol} & \textbf{Meaning} \\
  \midrule
  $D=\{(x_i,y_i)\}_{i=1}^n$ & Training dataset with $n$ labeled instances. \\
  $d_i=(x_i,y_i)$ & Individual training instance. \\
  $T$ & Number of sampled subsets / trained models in the MSR estimator. \\
  $m$ & Number of test or validation instances. \\
  $f_D, \theta_D$ & Model trained on dataset $D$ and its parameters. \\
  $\mathcal{U}$ & Utility function, e.g., accuracy. \\
  $V(d_i)$ & Scalar data value assigned to training instance $d_i$. \\
  $w \in \{0,1\}^n$ & Binary selection vector indicating which training instances are retained. \\
  $S$ & Retention budget, i.e., number of training instances kept. \\
  $\mathbf{T} \in \mathbb{R}^{n \times m}$ & Attribution matrix whose entry $\mathbf{T}_{ij}$ measures the influence of training instance $i$ on test instance $j$. \\
  $v = \mathbf{T}^\top w$ & Induced per-test utility vector under selection $w$. \\
  $t_j$ & Slack variable penalizing influence above the threshold $\kappa$. \\
  $\alpha$ & Trade-off parameter between total utility and slack penalty in CDVM. \\
  $\kappa$ & Soft upper bound on per-test influence in CDVM. \\
  $p$ & Inclusion probability used to sample training subsets in the MSR estimator. \\
  \midrule
  $D = C_1 \cup \dots \cup C_K$ & Clustered training set used in the theoretical toy model. \\
  $\mathcal{T} = \mathcal{T}_1 \cup \dots \cup \mathcal{T}_K$ & Analogous partition of the test set in the clustered model. \\
  $K$ & Number of clusters. \\
  $n_k = |C_k|$ & Size of training cluster $C_k$. \\
  $m_k = |\mathcal{T}_k|$ & Size of test cluster $\mathcal{T}_k$. \\
  $u_k$ & Total utility contributed by cluster $C_k$ when it is represented. \\
  $\tau_k$ & Cluster-specific attribution level in the block-structured attribution matrix used in the CDVM cluster analysis. \\
  $s_k = \sum_{i \in C_k} w_i$ & Number of selected training points from cluster $C_k$. \\
  \bottomrule
  \end{tabular}
  \caption{Summary of the main notation used in the paper.}
  \label{tab:notation_summary}
\end{table}

\newpage 
\section{Result Details}
\label{appendix:results}
We provide supplementary details for the main paper. Table~\ref{tab:results} tabulates the numerical results underlying the benchmark curves, while Figures~\ref{fig_app:cdvm_p_benchmark} and~\ref{fig_app:cdvm_runtime_benchmark} plot CDVM's performance and runtime across different values of $p$ and $T\,$.

\input{results_table}

\newpage 
\begin{figure}[!h]
    \centering
    \includegraphics[width=0.49\linewidth]{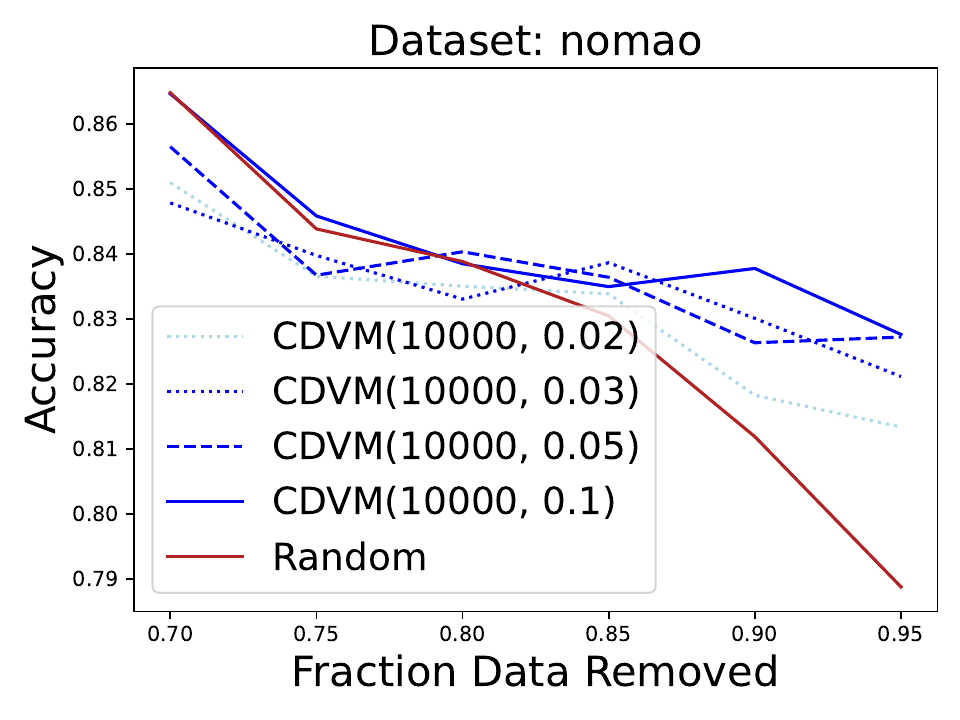}
    \includegraphics[width=0.49\linewidth]{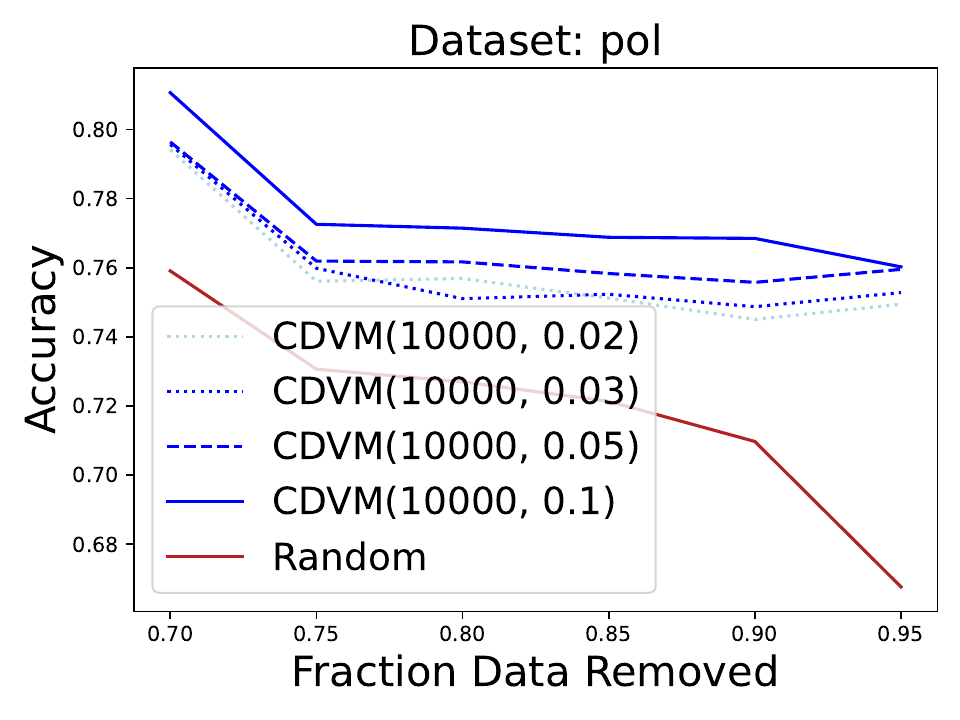}
    
    \includegraphics[width=0.49\linewidth]{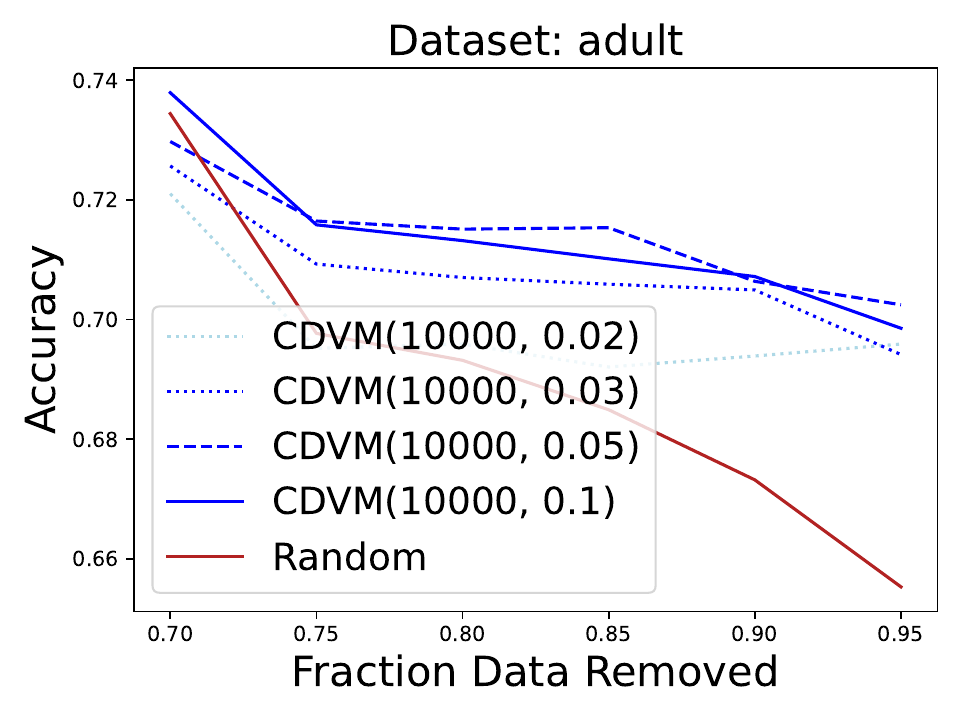}
    \includegraphics[width=0.49\linewidth]{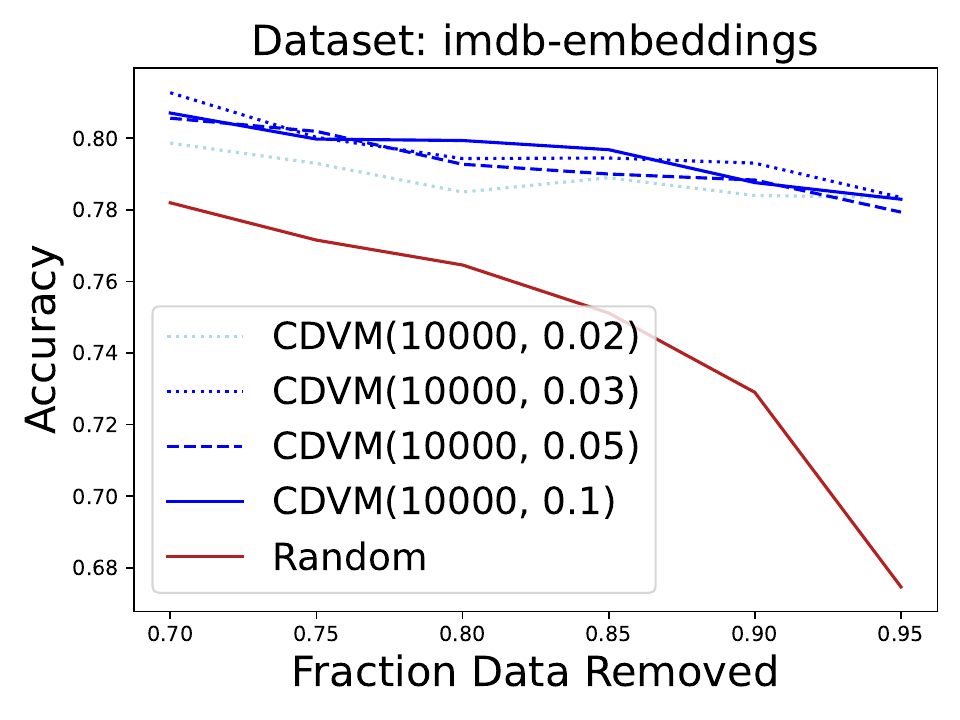}
    
    \includegraphics[width=0.49\linewidth]{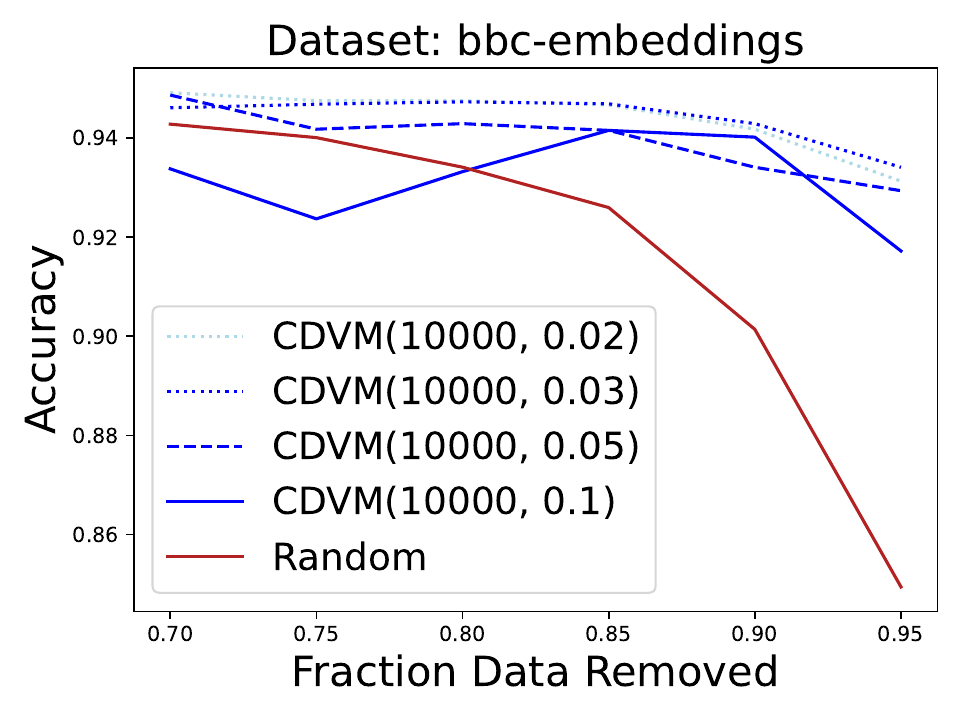}
    \includegraphics[width=0.49\linewidth]{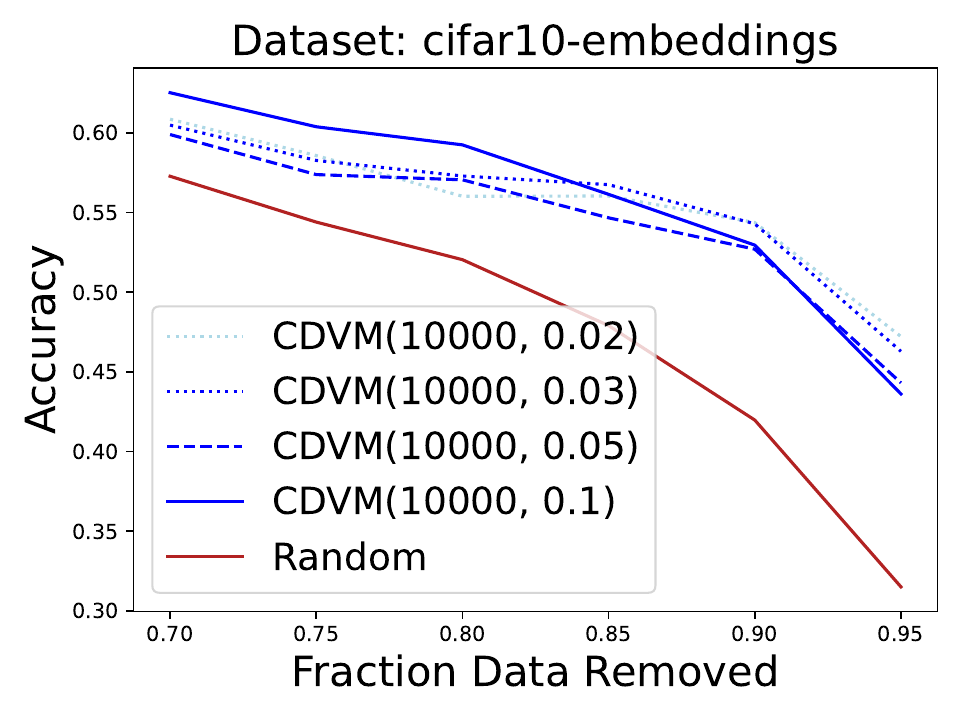}

    \caption{CDVM performance for different sampling probabilities \(p\in\{0.02,0.03,0.05,0.10\}\) on six datasets. 
    A higher sampling rate (\(p=0.10\)) yields the best pruning accuracy on Nomao, POL, and Adult, and outperforms lower \(p\) values up to 85\% removal on CIFAR-10, but degrades performance on BBC.
  }
    \label{fig_app:cdvm_p_benchmark}
\end{figure}

\newpage 
\begin{figure*}[!ht]
    \centering
    \includegraphics[width=0.49\linewidth]{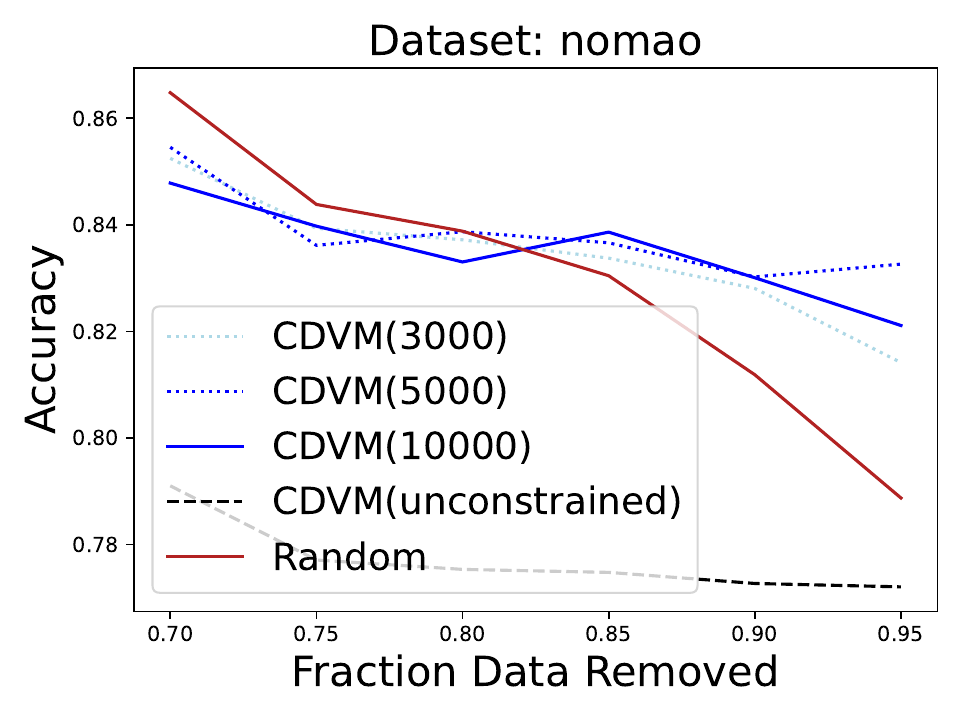}
    \includegraphics[width=0.49\linewidth]{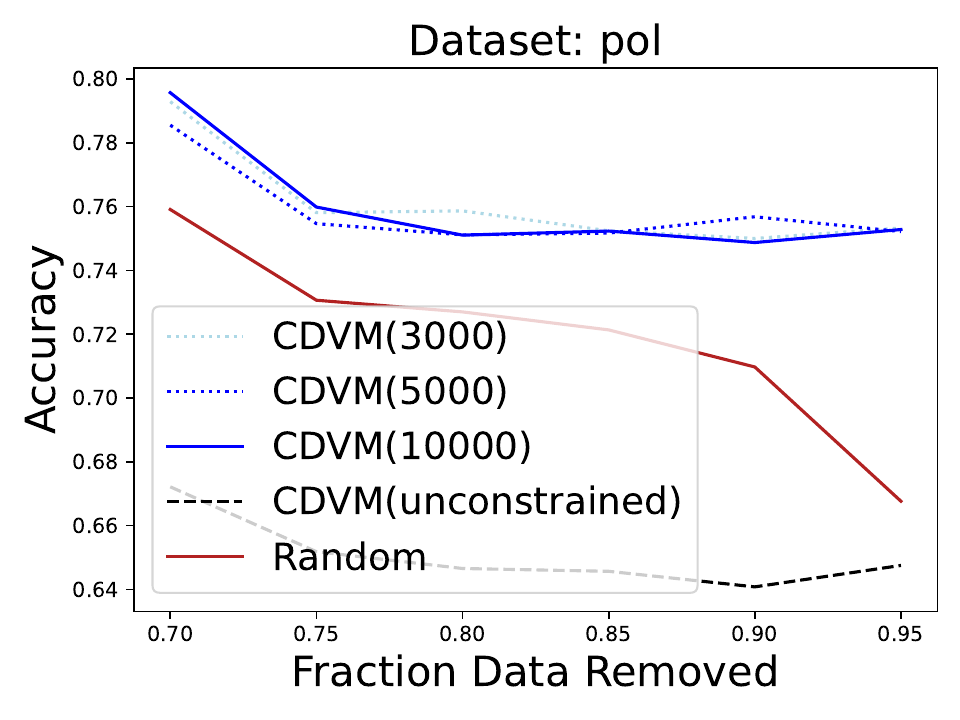}
    
    \includegraphics[width=0.49\linewidth]{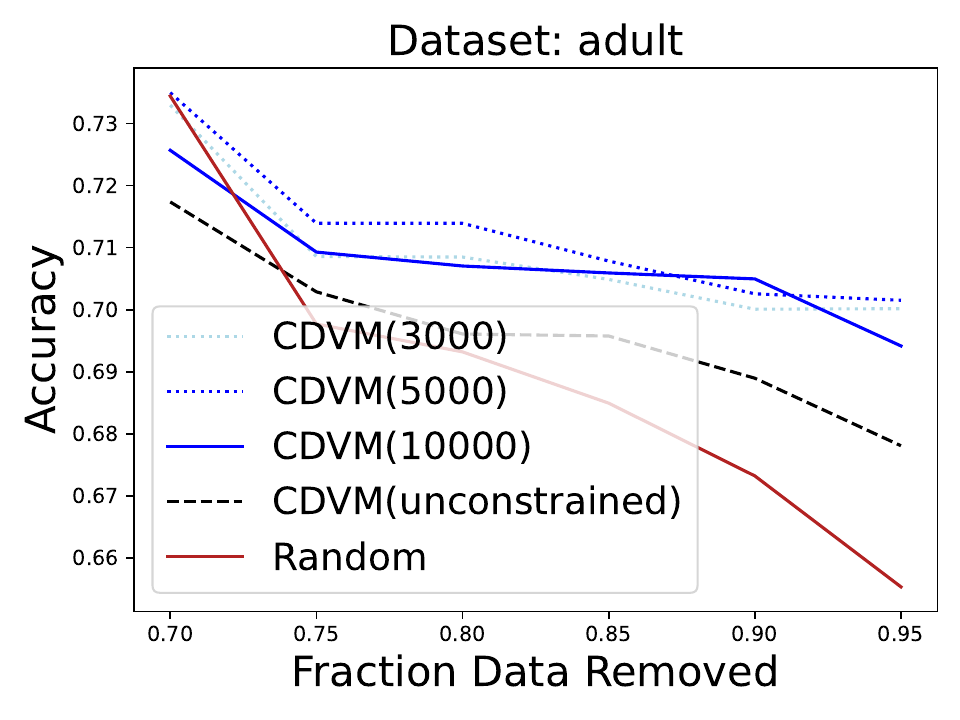}
    \includegraphics[width=0.49\linewidth]{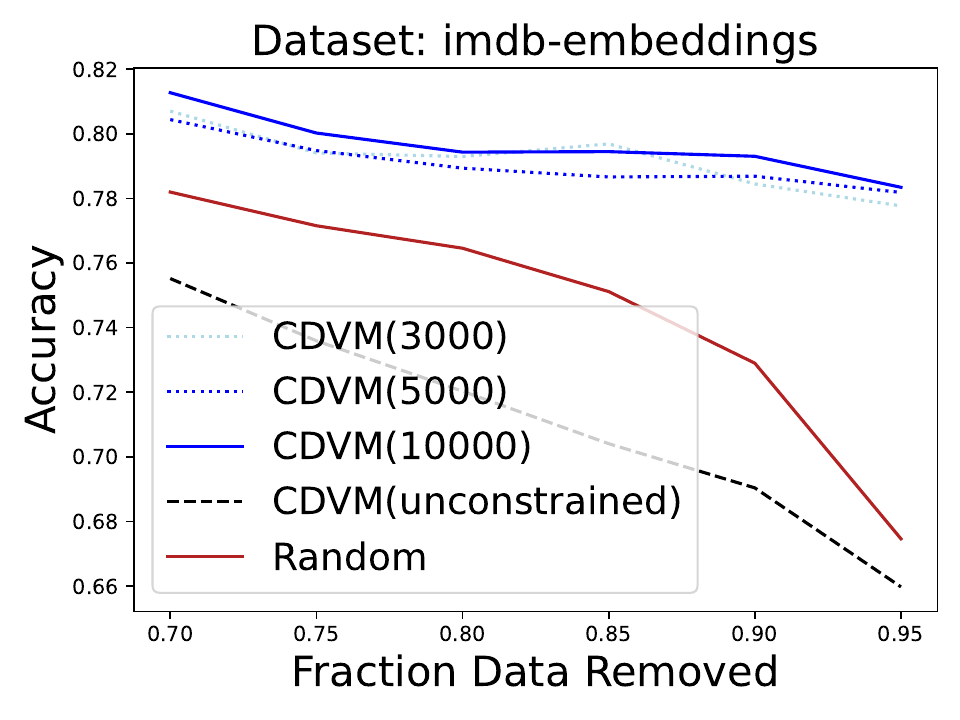}
    
    \includegraphics[width=0.49\linewidth]{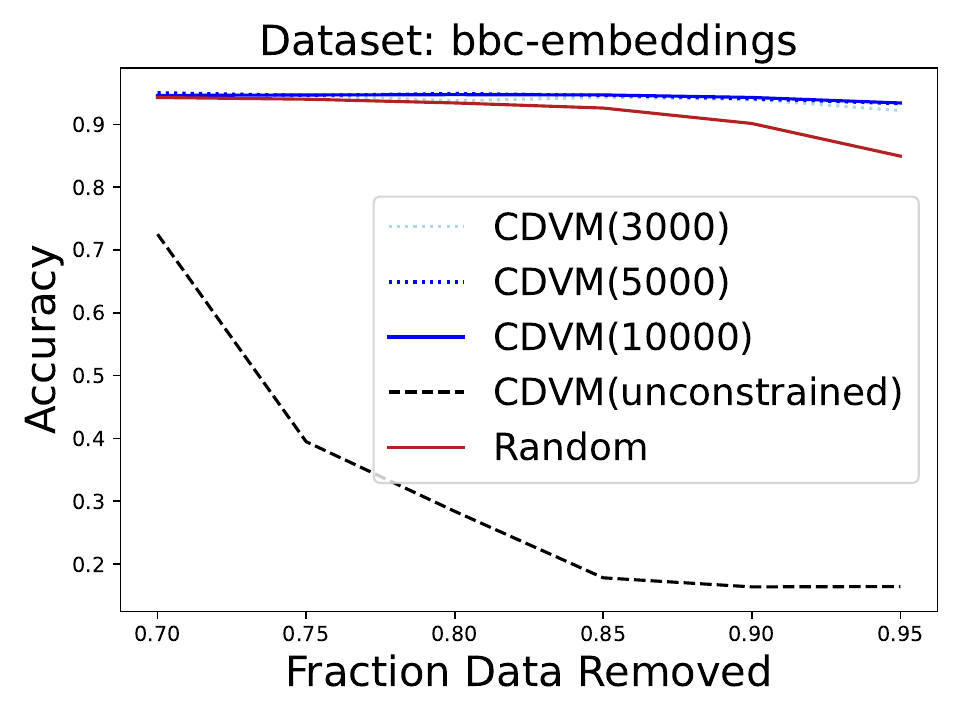}
    \includegraphics[width=0.49\linewidth]{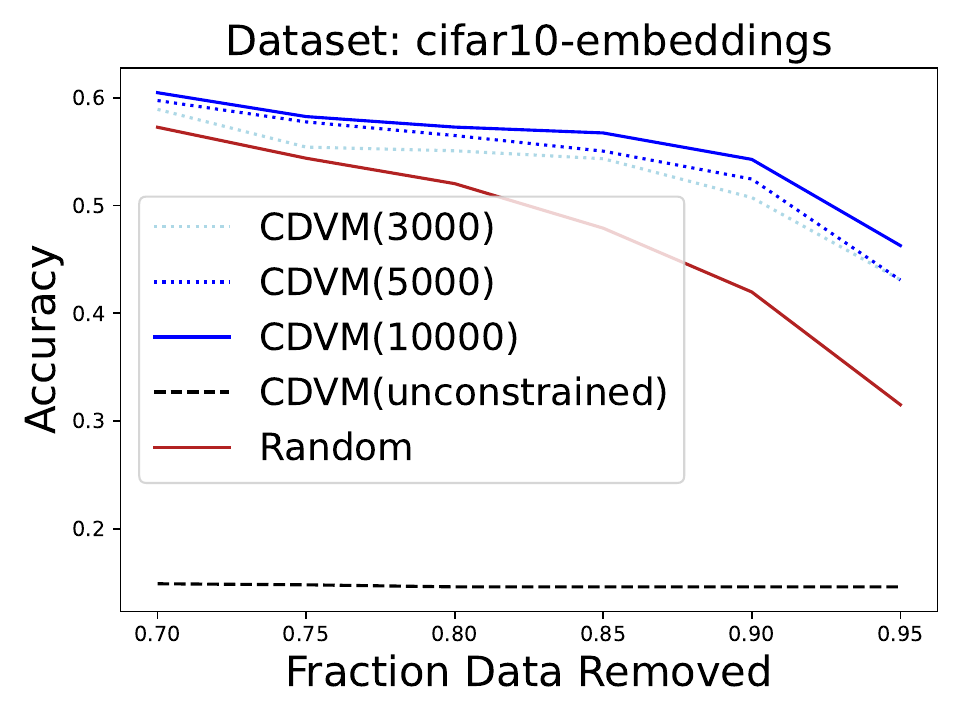}

    \caption{Effect of model-count and slack constraint on CDVM's runtime and accuracy. For each dataset, we compare CDVM using 3,000, 5,000, and 10,000 models to estimate the attribution matrix \(\mathbf{T}\), as well as a variant without the \(\kappa\) constraint. This unconstrained variant corresponds to DataBanzhaf computed from the same attribution estimates and mainly serves as a baseline showing that CDVM's gains come from the optimization rather than additional computation. In general, increasing the number of models improves pruning quality at the cost of longer runtime, while removing the slack constraint causes a severe drop in performance.}
    \label{fig_app:cdvm_runtime_benchmark}
\end{figure*}

\newpage 
\section{Ablation Study Details}
\label{app:ablation_study}
In addition, we detail our ablation study, particularly the performance normalization procedure, provide deeper insight into algorithm runtimes by distinguishing preparation and optimization times, analyze the persistence structure of retained instances across retention budgets, and demonstrate CDVM's scalability on a full dataset.

\subsection{Performance Normalization}
In Figure~\ref{fig:ablation}, we condense each method's performance across all evaluation settings into a single normalized score. To do so, we normalize each method's total score by the sum of the best- and worst-case performances.  Formally, let \(S\) be the set of all 36 evaluation settings (6 datasets $\times$ 6 pruning levels), and let \(p_{m,s}\) denote the test accuracy of method \(m\) on setting \(s\). Define

\[
P_m \;=\;\sum_{s \in S} p_{m,s}, 
\qquad
P_{\max} \;=\;\sum_{s \in S} \max_{m'} p_{m',s}, 
\qquad
P_{\min} \;=\;\sum_{s \in S} \min_{m'} p_{m',s}.
\]

Then the normalized performance of method \(m\) is

\[
\widetilde P_m
\;=\;
\frac{P_m - P_{\min}}{P_{\max} - P_{\min}}\,,
\]

which maps the aggregate score of each method into the interval \([0,1]\).

\subsection{Optimization and Runtime}
\begin{table}[!h]
    \centering
    \renewcommand{\arraystretch}{1.5}
    \begin{tabular}{lccc}
    
    \toprule
        \textbf{Method}   & 
        \textbf{Preparation Time} & 
        \textbf{Optimization Time} &
        \textbf{Total Time} \\
    
    \midrule
        DataOob   & 
        135 s (1,000 models)& 
        n.a. &
        135 s \\
        
        CDVM & 
        97 s (5,000 models)      
        & 10 s 
        & 107 s\\
        & 184 s (10,000 models)  & 
        & 194 s  \\
    
        InfOpt & 
        6 h (1,000 data instances)& 
        366 s (cifar-10)  
        & $\approx$ 6h \\
        & & 3 s (imdb) 
        & $\approx$ 6h \\
    
    \bottomrule
    \end{tabular}

    \caption{Runtime comparison (preparation + optimization). Wall-clock times for (i) DataOob/memorization, (ii) CDVM, and (iii) Influence-Function Optimization (InfOpt) on the OpenDataVal benchmark. DataOob uses bootstrap samples of size \(n\) (with replacement) and retrains \(T_{\mathrm{OOB}}=1{,}000\) models.  CDVM samples each training point with probability \(p=0.03\) and retrains \(T_{\mathrm{CDVM}}=5{,}000\) (or \(10{,}000\)) models to achieve stable estimates. Because each CDVM model sees only \(3\%\) of the data, individual training runs are much faster. InfOpt avoids retraining but must invert a Hessian per training instance and solve a quadratic program, resulting in multi-hour runtimes. Our method has constant optimization time because all datasets are scaled to the same size (1,000 training and 500 validation + test instances). For InfOpt, optimization time scales with the dataset's input dimension, whereas preparation time remains largely constant. }
    \label{tab:appendix_runtime}

\end{table}

\newpage 
\subsection{Frequency Spectrum of Training Instances Across Budgets}
\label{app:nestedness}

Figure~\ref{fig:ablation}(d) shows that off-diagonal overlap coefficients consistently exceed diagonal ones, indicating that larger retained subsets largely contain the smaller ones. This stands in apparent contrast to Figure~\ref{fig:empirical_examples} (center), which showed that optimal subsets are strongly budget-specific: accuracy peaks sharply at each target retention level and collapses immediately beyond it. To better understand this interplay, Figure~\ref{fig:appendix_nestedness} provides a fine-grained decomposition of all training instances by their selection frequency across budgets.

Specifically, the figure uses the
mean retention frequency $f(i, B)$ averaged across seeds as the primary
signal (e.g.\ $f(i, 5\%) = 0.6$ means that training point $i$ was selected
in at least 60\% of the 25 random seeds at retention level 5\%).
Each instance is first classified as either \emph{majority-selected}
($f(i, B) > 0.5$ for at least one budget) or \emph{sub-majority}
($f(i, B) \leq 0.5$ at every budget).
The majority-selected instances are further partitioned by persistence into
\emph{budget-specific} (majority at exactly one level),
\emph{useful pool, low} (majority at two to three levels),
\emph{useful pool, high} (majority at four to five levels), and
\emph{stable core} (majority at all six levels).
The sub-majority instances are decomposed by their peak frequency
$\max_B f(i, B)$ into four bands:
\emph{approaching} ($\max_B f \in [0.3, 0.5)$),
\emph{occasionally selected} ($[0.1, 0.3)$),
\emph{rarely selected} ($[0.01, 0.1)$), and
\emph{virtually never selected} ($< 0.01$).
The plot is sorted by the total majority-selected fraction.

\paragraph{Sub-majority dominates.}
On average, 84.6\% of instances (range: 77.7\%--95.9\%) remain
sub-majority across all retention levels.
Crucially, this group is \emph{not} inactive: within it,
the \emph{approaching} band alone accounts for a mean of 31.7\% of all
instances (up to 49.2\% for \textsc{pol}), meaning these instances are
frequently selected by a substantial minority of seeds but consistently fall
short of majority status.
A further 40.6\% on average are \emph{occasionally selected}
($\max_B f \in [0.1, 0.3)$).
Only 2.2\% of all instances are \emph{virtually never selected}
($\max_B f < 0.01$).

\paragraph{Majority-selected instances form a small, structured minority.}
Across datasets, only 4.1\% (\textsc{pol}) to 22.3\% (\textsc{imdb}) of
instances reach majority-selection status at any budget (mean: 15.4\%).
Within this minority, budget-specific instances account for
2.7\%--11.4\% (mean 6.8\%).
The useful pool splits further into a low-persistence group (majority at
two to three budgets: mean 6.2\%, range 0.8\%--11.0\%) and a
high-persistence group (majority at four to five budgets: mean 2.2\%,
range 0.3\%--4.1\%), indicating that cross-budget consistency within the
useful pool is itself graded rather than binary.
The stable core, instances majority-selected at \emph{all} six retention
levels, is negligibly small at 0.1\%--0.4\% (mean 0.2\%), confirming that
truly budget-invariant instances are virtually absent.

\paragraph{Dataset variation.}
\textsc{pol} is the clearest outlier: with only 4.1\% majority-selected
instances and 49.2\% of all instances in the approaching band, its solution
landscape is nearly flat, with many instances competing for retention slots
at similar frequencies.
\textsc{imdb} occupies the opposite extreme, with 22.3\% majority-selected
instances and the largest useful pool overall (low: 11.0\%, high: 4.1\%).

Taken together, the decomposition indicates that budget-induced differences
in the retained set affect a small minority of instances
(budget-specific: mean 6.8\%) rather than the population at large.
Moreover, the negligible stable core (mean 0.2\%) confirms that no
meaningful set of instances is consistently majority-selected across all
retention levels.

\begin{figure}[!ht]
    \centering
    \includegraphics[width=.65\textwidth]{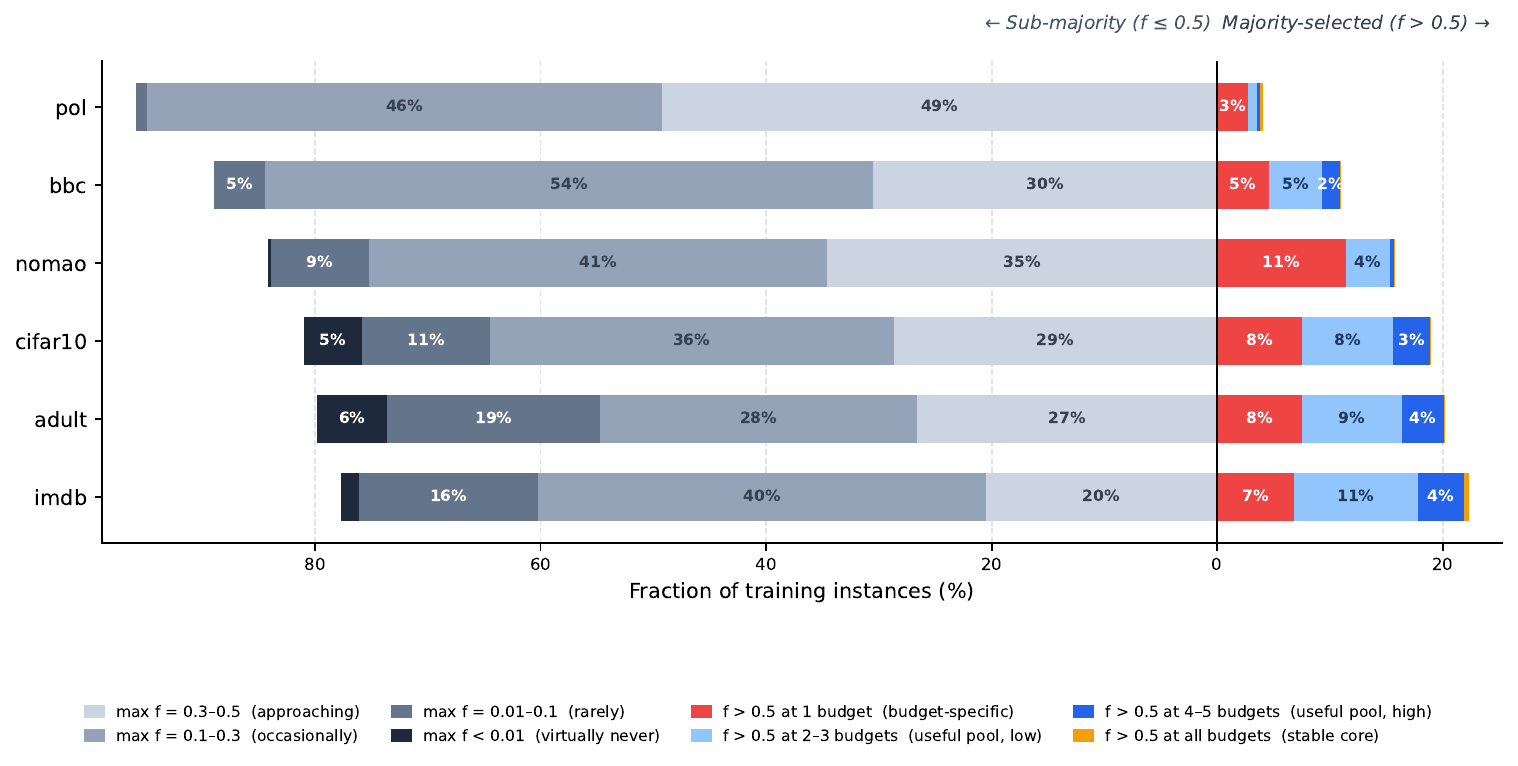}
    \caption{
        Selection frequency spectrum of training instances, sorted by total
        majority-selected fraction (highest at top).
        Each bar is divided at the majority threshold $f = 0.5$:
        the \textbf{left side} decomposes sub-majority instances
        ($f \leq 0.5$ at all budgets) by peak frequency $\max_B f(i,B)$
        into four bands, approaching majority (light grey, $[0.3, 0.5)$),
        occasionally selected ($[0.1, 0.3)$), rarely selected
        ($[0.01, 0.1)$), and virtually never selected (dark, ${<}0.01$);
        the \textbf{right side} decomposes majority-selected instances
        ($f > 0.5$ at $\geq\!1$ budget) by persistence into
        budget-specific (red, exactly one level),
        useful pool low (light blue, two to three levels),
        useful pool high (dark blue, four to five levels),
        and stable core (gold, all six levels).
        On average, 84.6\% of instances are sub-majority, of which only
        2.2\% are virtually never selected;
        the majority-selected minority (mean 15.4\%) comprises
        budget-specific instances (6.8\%),
        a low-persistence useful pool (6.2\%),
        a high-persistence useful pool (2.2\%),
        and a negligible stable core (0.2\%).
    }
    \label{fig:appendix_nestedness}
\end{figure}

\newpage 
\subsection{Scalability of CDVM to Larger Datasets}
\label{app:scaling}

\begin{figure*}[!ht]
    \centering
    \includegraphics[width=0.32\linewidth]{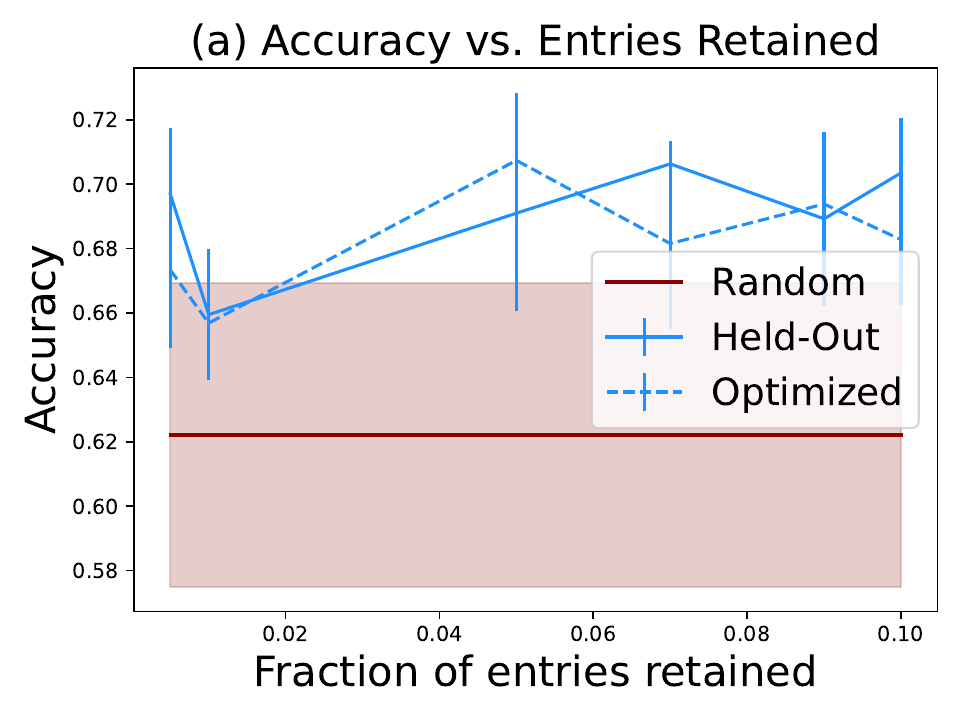}
    \includegraphics[width=0.32\linewidth]{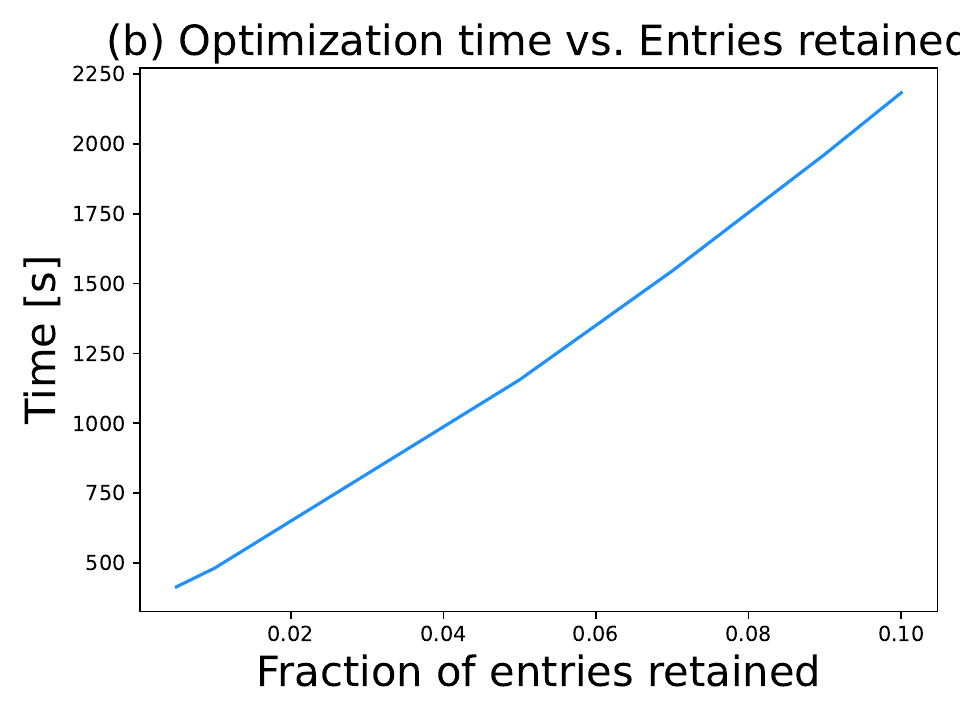}
    \includegraphics[width=0.32\linewidth]{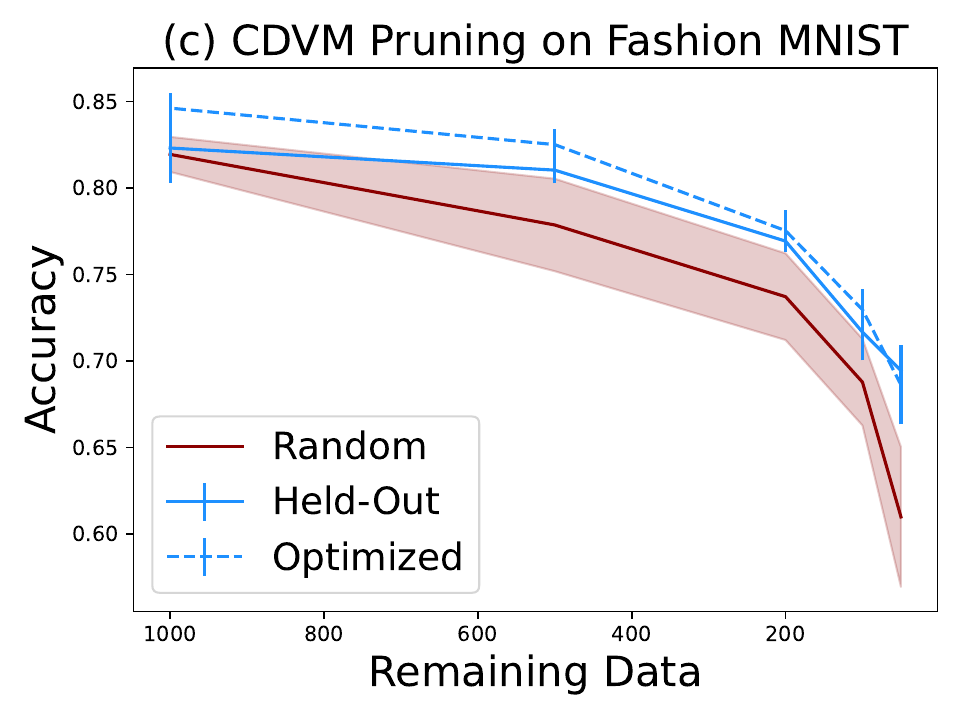}
    \caption{\textbf{(a)} Test accuracy for different sparsity cut-offs in the attribution matrix, expressed as the percentage of entries retained by not setting them to zero. 
    \textbf{(b)} Runtime (in seconds) for solving the CDVM optimization on each corresponding sparse matrix. 
    \textbf{(c)} CDVM test performance compared against a random-pruning baseline.}
    \label{fig:appendix_scaling}
\end{figure*}

So far, we have evaluated CDVM on the OpenDataVal benchmark by subsampling each split to 1\,000/500/500 (train/validation/test), enabling rapid comparisons across methods. To assess scalability on a larger attribution matrix, we apply CDVM to Fashion-MNIST (60\,000 training and 10\,000 test images). Fashion-MNIST was chosen because it offers a sizable dataset while still permitting fast model training. As before, the resulting attribution matrix $\mathbf{T}$ is highly sparse: many training instances have zero or negligible influence on most test samples. However, due to numerical noise, most entries remain small nonzero values and must be filtered out.

To study the effect of this residual noise, we threshold $\mathbf{T}$ by retaining only the top 0.5\,\%--10\,\% of its entries and setting all others to zero. The corresponding test accuracies are shown in Figure~\ref{fig:appendix_scaling}. We split the test set into an evaluation partition, for selecting training examples and tuning hyperparameters, and a held-out test partition for final performance assessment, and report results on both.

As shown, retaining just 5\,\% of the examples suffices to achieve high accuracy; including more examples yields no further benefit. Estimating $\mathbf{T}$ requires retraining between 5\,000 and 8\,000 models, which takes approximately 1--10 hours on a standard workstation, depending on the subsampling rate~$p$.

\newpage
\section{Cluster-Size Bias of Shapley and Banzhaf Data Values}
\label{sec:app_cluster_bias}

Let the training set be partitioned into $K$ disjoint clusters
\[
D \;=\; C_1 \cup \dots \cup C_K, 
\qquad C_k \cap C_\ell = \emptyset \text{ for } k\neq \ell,
\]
with $|C_k| = n_k \ge 1$ and total size $n = \sum_{k=1}^K n_k$.
We consider a cooperative game $(D, v)$ with characteristic function
\[
v(S)
\;=\;
\sum_{k=1}^K u_k \,\mathbf{1}\{S \cap C_k \neq \emptyset\},
\qquad S \subseteq D,
\]
where $u_k > 0$ is the contribution (utility) of cluster $C_k$, and
$\mathbf{1}\{\cdot\}$ is the indicator function.

\subsection{Shapley Data Values}
\label{sec:app_bias_shapley}

For any training instance $i \in C_k$, the Shapley value is
\[
\phi_i^{\text{Shapley}} \;=\; \frac{u_k}{n_k}.
\]
In particular, if all clusters have equal contribution $u_k = u$, then
\[
\phi_i^{\text{Shapley}} \;=\; \frac{u}{n_k}, \qquad i \in C_k,
\]
so points in larger clusters receive strictly smaller Shapley values.

\begin{proof}[Proof sketch]
Recall the permutation definition of the Shapley value:
\[
\phi_i
\;=\;
\mathbb{E}_{\pi}\!\left[
  v\bigl(S_i(\pi) \cup \{i\}\bigr) - v\bigl(S_i(\pi)\bigr)
\right],
\]
where $\pi$ is a uniformly random permutation of $D$ and
$S_i(\pi)$ is the set of players appearing before $i$ in $\pi$.

Fix a cluster $C_k$ and a point $i \in C_k$.  
By the definition of $v$, the marginal contribution of $i$ is
\[
v\bigl(S_i(\pi) \cup \{i\}\bigr) - v\bigl(S_i(\pi)\bigr)
=
\begin{cases}
u_k, & \text{if } S_i(\pi) \cap C_k = \emptyset,\\[1ex]
0,   & \text{otherwise.}
\end{cases}
\]
That is, $i$ contributes $u_k$ if and only if it is the \emph{first} point from
its cluster in the permutation $\pi$.
Because $\pi$ is uniform and all $n_k$ members of $C_k$ are symmetric within the
cluster, the probability that $i$ is the first point of $C_k$ in $\pi$ is
\[
\mathbb{P}\bigl(i \text{ is first in } C_k\bigr) = \frac{1}{n_k}.
\]
Hence
\[
\phi_i^{\text{Shapley}}
=
\mathbb{E}_{\pi}\bigl[
  v(S_i(\pi) \cup \{i\}) - v(S_i(\pi))
\bigr]
=
u_k \cdot \mathbb{P}\bigl(i \text{ is first in } C_k\bigr)
=
\frac{u_k}{n_k},
\]
which simplifies to $\phi_i^{\text{Shapley}} = u/n_k$ when $u_k = u$.
\end{proof}

\newpage 
\subsection{Data Banzhaf Values}
\label{sec:app_bias_banzhaf}

Data Banzhaf for a point $i$ is defined as
\[
\phi_i^{\text{Banzhaf}}
\;=\;
\frac{1}{2^{n-1}}
\sum_{S \subseteq D\setminus \{i\}}
\Bigl( v(S\cup\{i\}) - v(S) \Bigr).
\]

For every training instance $i \in C_k$ we obtain
\[
\phi_i^{\text{Banzhaf}} \;=\; \frac{u_k}{2^{\,n_k-1}}.
\]

\begin{proof}[Proof sketch]
Fix $i \in C_k$. By the definition of $v$,
\[
v(S\cup\{i\}) - v(S)
=
\begin{cases}
u_k, & \text{if } S \cap C_k = \emptyset,\\[1ex]
0,   & \text{otherwise.}
\end{cases}
\]
Thus only coalitions $S$ that contain no point from $C_k$ contribute to the sum.

The set $D \setminus C_k$ contains $n - n_k$ elements, and any subset of these
elements yields a coalition $S$ with $S \cap C_k = \emptyset$. Hence the number
of such coalitions is
\[
\bigl|\{ S \subseteq D\setminus\{i\} : S \cap C_k = \emptyset \}\bigr|
= 2^{\,n - n_k}.
\]
For each of these coalitions, the marginal contribution is $u_k$, and for all
others it is zero. Therefore,
\[
\phi_i^{\text{Banzhaf}}
=
\frac{1}{2^{n-1}} \cdot 2^{\,n - n_k} \cdot u_k
=
\frac{u_k}{2^{\,n_k-1}}.
\]
\end{proof}

If all clusters have equal contribution $u_k = u$, then
\[
\phi_i^{\text{Banzhaf}} = \frac{u}{2^{\,n_k-1}}, \qquad i \in C_k,
\]
which decreases exponentially in the cluster size $n_k$.

\newpage
\subsection{Interpretation and Equal Train--Test Distributions}
\label{subsec:app_interpretation_and_equal_train_test}

If all clusters share the same utility $u_k = u$, then both Data Shapley and
Data Banzhaf assign smaller values to points in larger clusters. For any
$i \in C_k$,
\[
\phi_i^{\text{Shapley}} = \frac{u}{n_k}
\quad\text{and}\quad
\phi_i^{\text{Banzhaf}} = \frac{u}{2^{\,n_k-1}}.
\]
Both expressions are strictly decreasing in $n_k$, so although the absolute
scales differ, Shapley and Banzhaf induce the \emph{same ranking} of points:
instances from smaller clusters are always ranked above instances from larger
clusters.

More generally, the formulas
\[
\phi_i^{\text{Shapley}} = \frac{u_k}{n_k},
\qquad
\phi_i^{\text{Banzhaf}} = \frac{u_k}{2^{\,n_k-1}}
\]
admit a common interpretation. The numerator $u_k$ captures the \emph{total
utility} of cluster $C_k$, while the denominator encodes its \emph{redundancy}
in the training set. For Shapley, the value $u_k/n_k$ corresponds to sharing the
cluster's total contribution equally among its $n_k$ members: the more training
points play the same role, the less each individual point is worth. For Banzhaf,
the exponential factor $2^{n_k-1}$ reflects that, in the power-index setting,
the larger the cluster, the more coalitions already contain some member of $C_k$,
rendering $i$'s marginal contribution smaller.

\paragraph{Utility as test performance.}
Let the test set be partitioned analogously as
\[
  \mathcal{T} = \mathcal{T}_1 \cup \cdots \cup \mathcal{T}_K, \quad |\mathcal{T}_k| = m_k
\]
so that $\mathcal{T}_k$ contains the test points associated with cluster $C_k$, where a test point in $\mathcal{T}_k$ is predicted correctly if and
only if the training set contains at least one point from $C_k$. If the cluster utility
\(u_k\) measures, for instance, the contribution to overall accuracy, then it
is proportional to the number of test points in that cluster, \(m_k\). We can
write
\[
u_k = \lambda_1 m_k,
\]
with \(\lambda_1 = \tfrac{1}{m}\) for standard (normalized) accuracy, where
\(m = \sum_{k=1}^K m_k\) is the total number of test points.

Under this parametrization, the per-point values become
\[
\phi_i^{\text{Shapley}} = \frac{u_k}{n_k}
= \frac{\lambda_1 m_k}{n_k},
\qquad
\phi_i^{\text{Banzhaf}} = \frac{u_k}{2^{\,n_k-1}}
= \frac{\lambda_1 m_k}{2^{\,n_k-1}}.
\]
Hence, a point $i \in C_k$ becomes more valuable if its cluster matters for
many test instances (large $m_k$), but less valuable if many redundant training
points share the same role (large $n_k$).

\paragraph{Equal train--test distributions.}
An instructive edge case arises when train and test are \emph{equally
distributed} over clusters. This means there exists a constant $\lambda_2 > 0$
such that
\[
m_k = \lambda_2 n_k
\quad\text{for all } k = 1,\dots,K.
\]
Substituting into the expressions above yields, for any $i \in C_k$,
\[
\phi_i^{\text{Shapley}}
= \frac{\lambda_1 m_k}{n_k}
= \frac{\lambda_1 \lambda_2 n_k}{n_k}
= \lambda_1 \lambda_2,
\]
which is \emph{independent} of the cluster index $k$. Thus, when train and test
follow the same cluster distribution (even if their absolute sizes differ),
Shapley assigns the same value to every training instance: all points are
equally important because each cluster's influence on the test set is exactly
matched by its representation in the training set.

For Data Banzhaf, we obtain
\[
\phi_i^{\text{Banzhaf}}
= \frac{\lambda_1 m_k}{2^{\,n_k-1}}
= \frac{\lambda_1 \lambda_2\, n_k}{2^{\,n_k-1}}.
\]
The factor $n_k/2^{\,n_k-1}$ is non-increasing in $n_k$ (and strictly decreasing
for $n_k \ge 2$), so even under equal train--test distributions, Banzhaf
\emph{does not} cancel the dependence on cluster size: points from smaller
clusters receive strictly larger values than points from larger clusters.

\paragraph{Implications for pruning.}
In this clustered setting, both notions are biased for data pruning:

\begin{itemize}
    \item Under equal train--test distributions, Shapley assigns identical
          values to all points, so any pruning strategy based solely on Shapley
          values reduces to an essentially arbitrary order.
    \item Data Banzhaf continues to favor points in small clusters and penalize
          large clusters exponentially. A Banzhaf-based pruning strategy thus
          tends to remove points from large clusters first, even when keeping at
          least one representative per cluster would be optimal from a test-utility
          perspective.
\end{itemize}

\newpage 
\section{CDVM in a Clustered Setting}
\label{sec:app_cdvm_cluster}

We now show that, in the same clustered scenario used to analyze Shapley
and Banzhaf data values, CDVM behaves qualitatively differently:
it avoids pruning entire clusters as long as the pruning budget
permits.

\subsection{CDVM Objective in the Cluster Model}

As before, let the training set be partitioned into $K$ disjoint clusters
\[
D \;=\; C_1 \cup \dots \cup C_K,
\qquad C_k \cap C_\ell = \emptyset \text{ for } k \neq \ell,
\]
with $|C_k| = n_k$ and total size $n = \sum_{k=1}^K n_k$.
Let the test set be partitioned analogously as
\[
\mathcal{T} = \mathcal{T}_1 \cup \cdots \cup \mathcal{T}_K,
\quad |\mathcal{T}_k| = m_k,
\]
so that $\mathcal{T}_k$ contains the test points associated with cluster $C_k$.

The attribution matrix $\mathbf{T} \in \mathbb{R}^{n \times m}$ for this setup is block-structured of the form
\[
\mathbf{T}_{ij} =
\begin{cases}
\tau_k > 0, & \text{if } i \in C_k,\; j \in \mathcal{T}_k,\\[0.4ex]
0,          & \text{otherwise,}
\end{cases}
\]
i.e., each training point in cluster $C_k$ contributes the same
positive amount $\tau_k$ to each test point in $\mathcal{T}_k$ and nothing to
other clusters.

Let $w \in \{0,1\}^n$ be the selection vector and
\[
  v = \mathbf{T}^\top w \in \mathbb{R}^m
\]
the induced influence on test points. In this clustered setting,
\[
  v_j = \tau_k\,s_k
  \quad\text{for all } j \in \mathcal{T}_k,
  \qquad
  s_k := \sum_{i \in C_k} w_i,
\]
i.e., $s_k$ is the number of selected training points from cluster $C_k$.

For ease of exposition, we work with the simplified CDVM surrogate obtained
by setting $\alpha = 0.5$ and eliminating slack variables (cf.\ the main
objective in Section~\ref{sec:cdvm}): this yields the problem
\[
  \max_{w \in \{0,1\}^n}
  \sum_{j=1}^m \min\{\kappa, v_j\}
  \quad\text{s.t.}\quad
  \sum_{i=1}^n w_i = S,
\]
where $S$ is the pruning budget (number of points kept) and $\kappa > 0$
is the slack threshold. In the clustered model, the objective becomes
\[
  \sum_{j=1}^m \min\{\kappa, v_j\}
  =
  \sum_{k=1}^K m_k \min\{\tau_k s_k, \kappa\}.
\]
Thus CDVM reduces to a \emph{cluster coverage} objective: each cluster
contributes $m_k \min\{\tau_k s_k, \kappa\}$ as soon as at least one
of its training points is kept ($s_k \ge 1$), and additional points from
the same cluster yield only extra gain as long as the product $\tau_k s_k$
has not yet saturated at $\kappa$.

\newpage
\subsection{CDVM Avoids Pruning Entire Clusters}
Under the clustered utility model above, an \emph{optimal} pruning strategy
removes all but one point from each cluster before removing any cluster's final
representative, and only starts emptying clusters once the budget no longer
allows keeping one point per cluster.

We now show that CDVM exhibits the same qualitative behavior: as long as the
budget \(S\) allows keeping one point per cluster and \(\kappa\) is chosen
appropriately, CDVM never removes an entire cluster.  

Set 
\[
  \kappa = \kappa_\tau := \min_{k} \tau_k.
\]
Then for any cluster $C_k$ and any $s_k \ge 1$,
\[
  \min\{\tau_k s_k, \kappa_\tau\}
  =
  \kappa_\tau,
\]
so that the objective simplifies to
\[
  \sum_{j=1}^m \min\{\kappa_\tau, v_j\}
  =
  \sum_{k=1}^K m_k \kappa_\tau \,\mathbf{1}\{s_k \ge 1\}.
\]
In this regime, only the question \emph{whether} a cluster is represented
matters; keeping more than one point per cluster brings no additional gain.

\begin{lemma}[CDVM preserves clusters when $S \ge K$]
\label{lem:cdvm_preserves_clusters}
Assume $\kappa = \kappa_\tau$ and $S \ge K$. Then for any optimal
solution $w^\star$ of the CDVM surrogate above,
\[
  \sum_{i \in C_k} w_i^\star \;\ge\; 1,
  \qquad k = 1,\dots,K.
\]
That is, as long as the budget $S$ allows keeping at least one point
per cluster, CDVM never removes an entire cluster.
\end{lemma}

\begin{proof}[Proof sketch]
Suppose there exists an optimal solution $w^\star$ and a cluster $C_a$
with $s_a^\star = 0$ (no selected point) while some other cluster $C_b$
has $s_b^\star \ge 2$. 
Move one selected point from $C_b$ to $C_a$, obtaining a new selection
$\widetilde{w}$ with $\sum_i \widetilde{w}_i = S$, $s_a(\widetilde{w}) = 1$,
and $s_b(\widetilde{w}) \ge 1$. Under $\kappa_\tau$,
\[
  m_b \kappa_\tau \,\mathbf{1}\{s_b^\star \ge 1\}
  =
  m_b \kappa_\tau \,\mathbf{1}\{s_b(\widetilde{w}) \ge 1\},
\]
so the contribution from $C_b$ is unchanged, while the contribution from
$C_a$ increases from $0$ to $m_a \kappa_\tau > 0$. All other clusters
are unaffected. Hence the objective strictly increases, contradicting the
optimality of $w^\star$. Therefore no optimal solution can leave a cluster
empty while another cluster contains more than one selected point. For
$S \ge K$, it is feasible to assign at least one point to each cluster,
so every optimal solution must represent all clusters.
\end{proof}

\newpage
\section{Synthetic Dataset}
\label{appendix:synthetic_data}
In this section, we provide further details and empirical examples on the synthetic dataset and Shapley-value data valuation, illustrating how interactions among data points can induce non-monotonic pruning behavior, for example, we construct a dataset in which removing more examples paradoxically improves accuracy, so that keeping less data can outperform keeping more.

\begin{figure*}[!ht]
    \centering
    \textbf{Full dataset}
    
    \includegraphics[width=.8\linewidth]{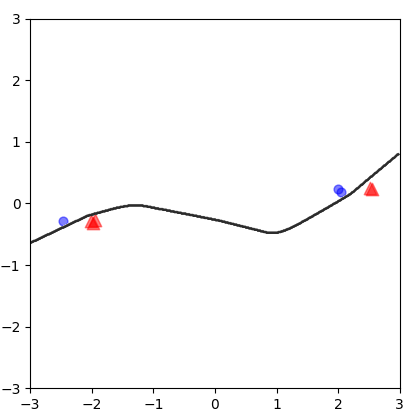}
    
\caption{Synthetic clustered dataset. The training set consists of eight
points from four Gaussian clusters
with centers \(\mu_1 = (-2, 0.5), \mu_2 = (2.5, 0)\) (red) and
\(\mu_3 = (-2.5, -0.5), \mu_4 = (2, 0)\) (blue), corresponding to clusters
\(C_1,\dots,C_4\) with sizes \(n_1 = 3\), \(n_2 = 2\), \(n_3 = 2\), and
\(n_4 = 1\). }
    \label{fig:app_syn_data_full}
\end{figure*}

\newpage
\subsection{Removing Entire Clusters}
Figure~\ref{fig:app_syn_entire_cluster} shows the decision boundary after removing each of the four clusters in turn. In every case the boundary shifts substantially, confirming that each cluster is necessary to correctly classify its associated test region.
\begin{figure*}[!ht]
    \centering
        \includegraphics[width=.49\linewidth]{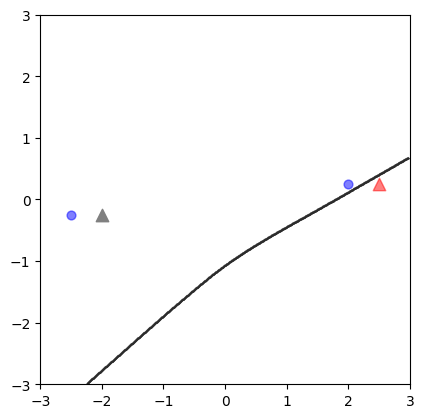}
        \includegraphics[width=.49\linewidth]{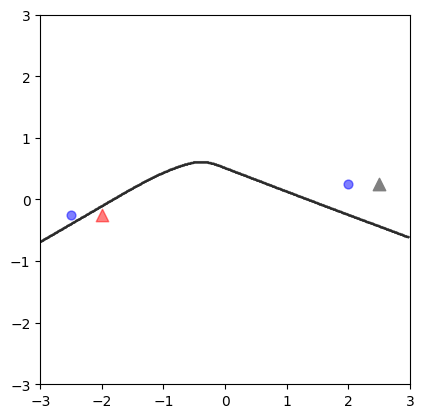}

        \includegraphics[width=.49\linewidth]{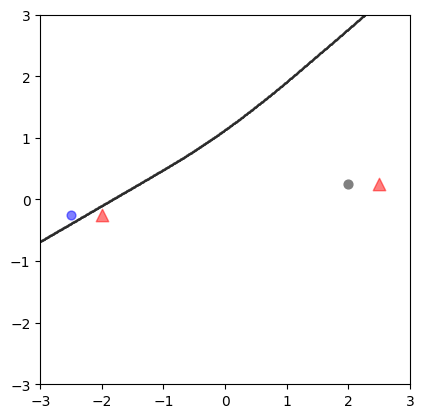}
        \includegraphics[width=.49\linewidth]{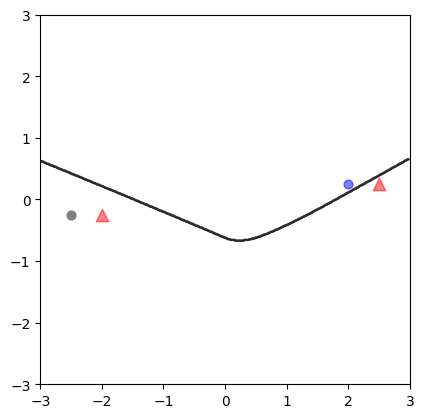}
        
    \caption{Effect of removing an entire cluster. The removed cluster is grey. }
    \label{fig:app_syn_entire_cluster}
\end{figure*}

\newpage
\subsection{Leave One Out}
\label{app:syn_loo}
Figure~\ref{fig:appendix_syn_loo} applies the leave-one-out procedure to each training point. Because all clusters except $C_4$ contain more than one point, removing any single member leaves the decision boundary unchanged and yields a zero LOO value. Only the singleton in $C_4$ produces a non-zero contribution.
\begin{figure*}[!ht]
    \textbf{LOO left red cluster}
    
    \includegraphics[width=.3\linewidth]{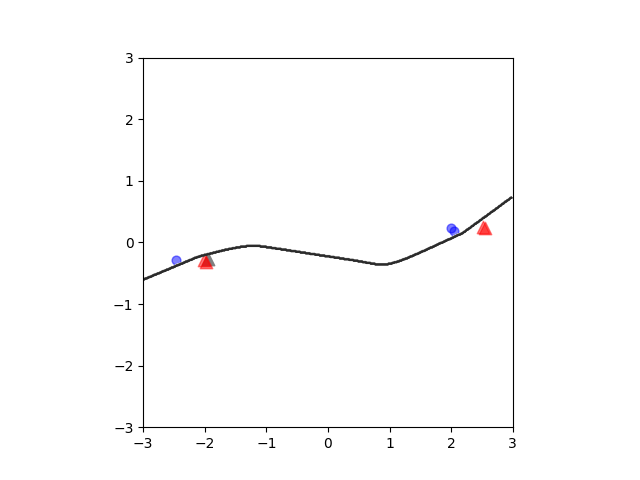}
    \includegraphics[width=.3\linewidth]{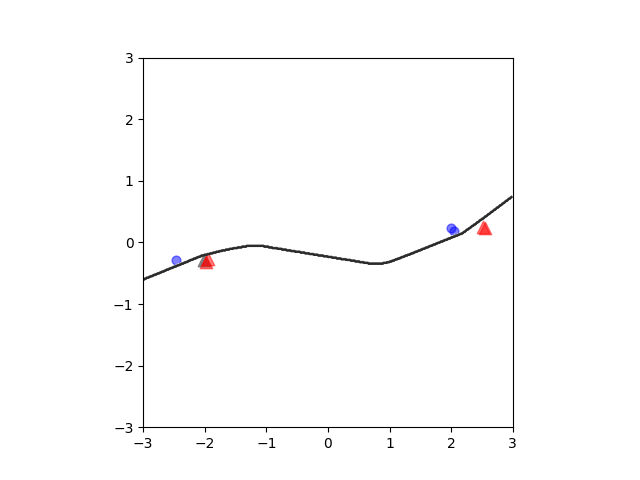}
    \includegraphics[width=.3\linewidth]{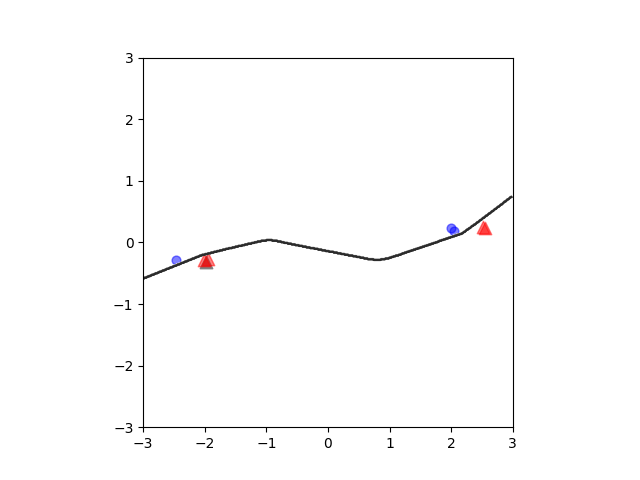}

    \textbf{LOO right red cluster}
    
    \includegraphics[width=.3\linewidth]{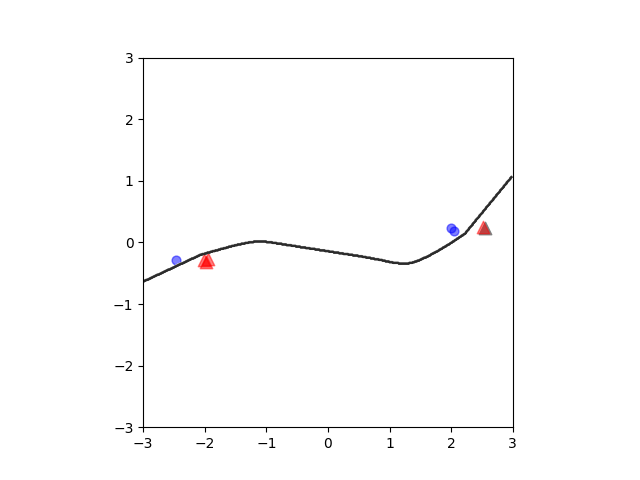}
    \includegraphics[width=.3\linewidth]{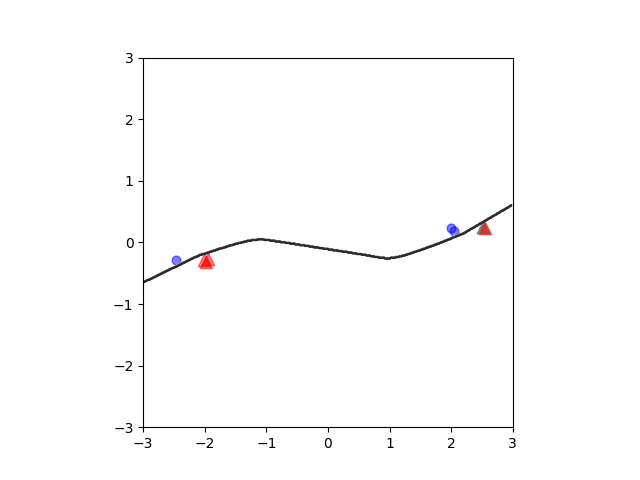}

    \textbf{LOO right blue cluster}
    
    \includegraphics[width=.3\linewidth]{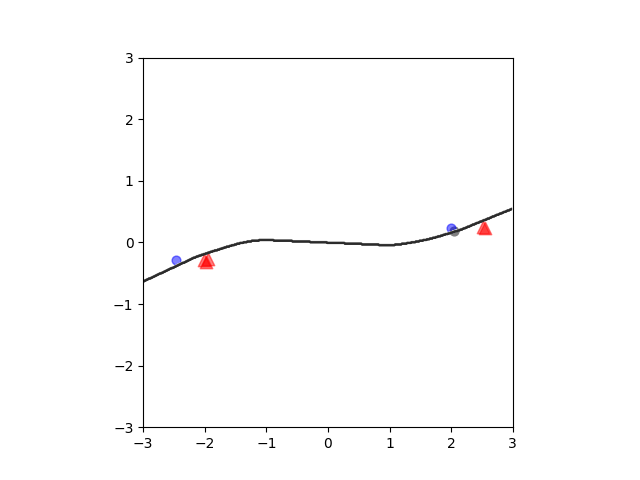}
    \includegraphics[width=.3\linewidth]{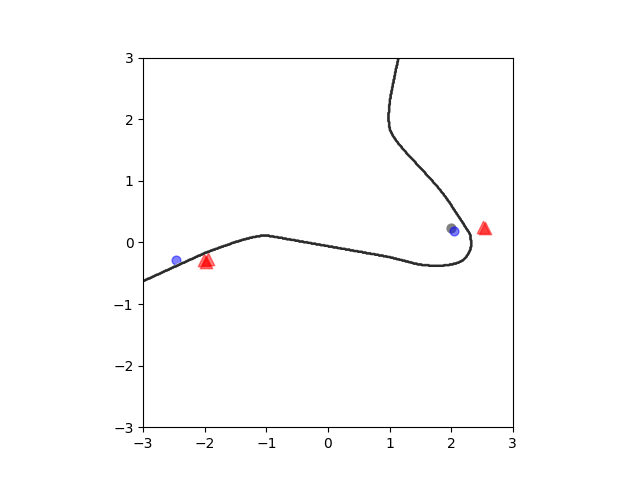}

    \textbf{LOO left blue cluster}
    
    \includegraphics[width=.3\linewidth]{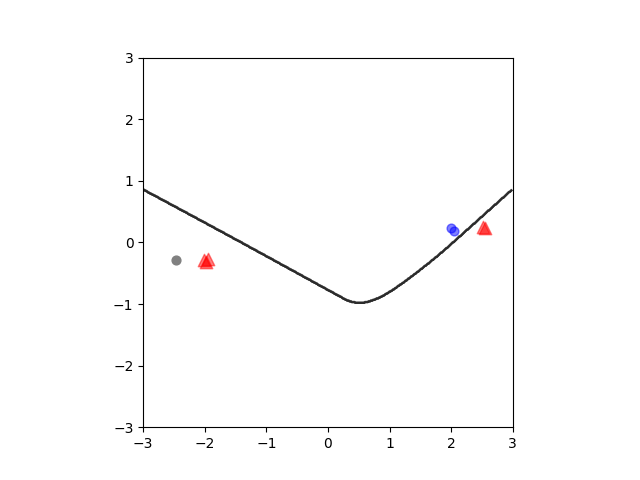}

    \caption{Leave-one-out (LOO) on the dataset from Figure \ref{fig:app_syn_data_full}. All clusters except the last contain more than one point; therefore, the decision boundary remains unchanged when a point is removed from these clusters. Each plot shows the effect of removing exactly one point from the respective cluster. 
    Consequently, only the point from the left blue cluster will exhibit a non-zero leave-one-out data value.}
    \label{fig:appendix_syn_loo}
\end{figure*}

\newpage
\subsection{Shapley Data Value}
\label{app:shap_sy_data}
Figure~\ref{fig:syn_shap_dv} shows the computed Shapley values for the synthetic dataset. Consistent with the cluster-size bias analysis in Appendix~\ref{sec:app_bias_shapley}, points in larger clusters receive smaller values, and the singleton in $C_4$ attains the highest value of all.
\begin{figure*}[!ht]
    \centering
    \includegraphics[width=.61\linewidth]{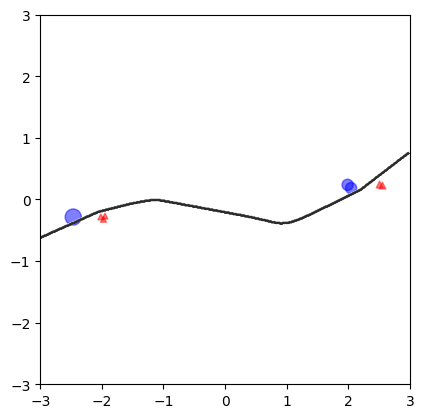}

    \caption{Shapley data valuation scores for the synthetic dataset. The black line represents the decision boundary of an MLP trained on this dataset. Given the 256 possible combinations for subsets, not all are plotted. Instead, the plot displays the computed data Shapley values, where the size of a point indicates its value. As observed, the values are proportional to the cluster size, with the blue singleton point exhibiting the highest value. }
    \label{fig:syn_shap_dv}
    
\end{figure*}

%% file: results_table.tex
\begin{table*}[h!]
    \centering
    \begin{tabular}{l|c c c|c c c}
        \hline 
        & 
            \textbf{30\% Data} &
            \textbf{25\% Data} &
            \textbf{20\% Data} &
            \textbf{15\% Data} &
            \textbf{10\% Data} &
            \textbf{5\% Data} \\
            
        \hline 
        \hline 
		\multicolumn{7}{c}{\textit{nomao}}\\
		\hline
			 CDVM5k & 
				 0.855 $\pm$ 0.02 & 
				 0.836 $\pm$ 0.02 & 
				 0.839 $\pm$ 0.02 & 
				 0.837 $\pm$ 0.02 & 
				 0.830 $\pm$ 0.02 & 
				 0.833 $\pm$ 0.03  
			 \\ 

			 CDVM10k & 
				 0.848 $\pm$ 0.01 & 
				 0.840 $\pm$ 0.02 & 
				 0.833 $\pm$ 0.02 & 
				 0.839 $\pm$ 0.02 & 
				 0.830 $\pm$ 0.02 & 
				 0.821 $\pm$ 0.02  
			 \\ 

			 CDVM-n & 
				 0.872 $\pm$ 0.02 & 
				 0.855 $\pm$ 0.02 & 
				 0.848 $\pm$ 0.02 & 
				 0.849 $\pm$ 0.02 & 
				 0.842 $\pm$ 0.02 & 
				 \textbf{0.842} $\pm$ 0.02  
			 \\ 

			 InfOpt & 
				 \textbf{0.874} $\pm$ 0.00 & 
				 \textbf{0.866} $\pm$ 0.00 & 
				 \textbf{0.862} $\pm$ 0.00 & 
				 \textbf{0.860} $\pm$ 0.00 & 
				 \textbf{0.858} $\pm$ 0.00 & 
				 0.808 $\pm$ 0.00  
			 \\ 

			 DataOOB & 
				 0.839 $\pm$ 0.01 & 
				 0.804 $\pm$ 0.01 & 
				 0.801 $\pm$ 0.01 & 
				 0.794 $\pm$ 0.01 & 
				 0.791 $\pm$ 0.01 & 
				 0.786 $\pm$ 0.02  
			 \\ 

			 Banzhaf & 
				 0.791 $\pm$ 0.02 & 
				 0.764 $\pm$ 0.02 & 
				 0.754 $\pm$ 0.02 & 
				 0.734 $\pm$ 0.03 & 
				 0.710 $\pm$ 0.03 & 
				 0.664 $\pm$ 0.07  
			 \\ 

			 Random & 
				 0.865 $\pm$ 0.02 & 
				 0.844 $\pm$ 0.02 & 
				 0.839 $\pm$ 0.02 & 
				 0.830 $\pm$ 0.03 & 
				 0.812 $\pm$ 0.03 & 
				 0.789 $\pm$ 0.04  
			 \\ 

		\hline
		\multicolumn{7}{c}{\textit{cifar10}}\\
		\hline
			 CDVM5k & 
				 0.598 $\pm$ 0.02 & 
				 0.578 $\pm$ 0.02 & 
				 0.565 $\pm$ 0.03 & 
				 0.551 $\pm$ 0.02 & 
				 0.525 $\pm$ 0.03 & 
				 0.431 $\pm$ 0.03  
			 \\ 

			 CDVM10k & 
				 \textbf{0.605} $\pm$ 0.02 & 
				 \textbf{0.583} $\pm$ 0.03 & 
				 \textbf{0.573} $\pm$ 0.03 & 
				 \textbf{0.567} $\pm$ 0.02 & 
				 \textbf{0.543} $\pm$ 0.02 & 
				 \textbf{0.463} $\pm$ 0.03  
			 \\ 

			 InfOpt & 
				 0.579 $\pm$ 0.02 & 
				 0.580 $\pm$ 0.00 & 
				 0.560 $\pm$ 0.00 & 
				 0.548 $\pm$ 0.00 & 
				 0.530 $\pm$ 0.00 & 
				 0.456 $\pm$ 0.00  
			 \\ 

			 DataOOB & 
				 0.570 $\pm$ 0.01 & 
				 0.495 $\pm$ 0.01 & 
				 0.490 $\pm$ 0.02 & 
				 0.458 $\pm$ 0.03 & 
				 0.413 $\pm$ 0.10 & 
				 0.322 $\pm$ 0.13  
			 \\ 

			 Banzhaf & 
				 0.592 $\pm$ 0.02 & 
				 0.570 $\pm$ 0.03 & 
				 0.522 $\pm$ 0.08 & 
				 0.494 $\pm$ 0.08 & 
				 0.438 $\pm$ 0.08 & 
				 0.332 $\pm$ 0.06  
			 \\ 

			 Random & 
				 0.573 $\pm$ 0.02 & 
				 0.544 $\pm$ 0.03 & 
				 0.520 $\pm$ 0.03 & 
				 0.479 $\pm$ 0.06 & 
				 0.420 $\pm$ 0.07 & 
				 0.315 $\pm$ 0.07  
			 \\ 

		\hline
		\multicolumn{7}{c}{\textit{pol}}\\
		\hline
			 CDVM5k & 
				 0.786 $\pm$ 0.02 & 
				 0.755 $\pm$ 0.03 & 
				 \textbf{0.751} $\pm$ 0.03 & 
				 0.752 $\pm$ 0.03 & 
				 \textbf{0.757} $\pm$ 0.03 & 
				 0.752 $\pm$ 0.04  
			 \\ 

			 CDVM10k & 
				 \textbf{0.796} $\pm$ 0.02 & 
				 \textbf{0.760} $\pm$ 0.03 & 
				 0.751 $\pm$ 0.03 & 
				 \textbf{0.752} $\pm$ 0.03 & 
				 0.749 $\pm$ 0.03 & 
				 \textbf{0.753} $\pm$ 0.03  
			 \\ 

			 InfOpt & 
				 0.699 $\pm$ 0.01 & 
				 0.608 $\pm$ 0.00 & 
				 0.604 $\pm$ 0.00 & 
				 0.572 $\pm$ 0.00 & 
				 0.472 $\pm$ 0.00 & 
				 0.350 $\pm$ 0.00  
			 \\ 

			 DataOOB & 
				 0.738 $\pm$ 0.03 & 
				 0.729 $\pm$ 0.04 & 
				 0.732 $\pm$ 0.03 & 
				 0.727 $\pm$ 0.03 & 
				 0.725 $\pm$ 0.04 & 
				 0.720 $\pm$ 0.04  
			 \\ 

			 Banzhaf & 
				 0.716 $\pm$ 0.04 & 
				 0.689 $\pm$ 0.04 & 
				 0.663 $\pm$ 0.05 & 
				 0.636 $\pm$ 0.05 & 
				 0.607 $\pm$ 0.04 & 
				 0.557 $\pm$ 0.06  
			 \\ 

			 Random & 
				 0.759 $\pm$ 0.03 & 
				 0.731 $\pm$ 0.03 & 
				 0.727 $\pm$ 0.03 & 
				 0.721 $\pm$ 0.04 & 
				 0.710 $\pm$ 0.04 & 
				 0.668 $\pm$ 0.05  
			 \\ 

		\hline
		\multicolumn{7}{c}{\textit{imdb}}\\
		\hline
			 CDVM5k & 
				 0.804 $\pm$ 0.01 & 
				 0.795 $\pm$ 0.02 & 
				 0.789 $\pm$ 0.02 & 
				 0.787 $\pm$ 0.02 & 
				 0.787 $\pm$ 0.02 & 
				 0.782 $\pm$ 0.02  
			 \\ 

			 CDVM10k & 
				 \textbf{0.813} $\pm$ 0.01 & 
				 0.800 $\pm$ 0.02 & 
				 \textbf{0.794} $\pm$ 0.03 & 
				 \textbf{0.794} $\pm$ 0.01 & 
				 \textbf{0.793} $\pm$ 0.01 & 
				 \textbf{0.783} $\pm$ 0.02  
			 \\ 

			 InfOpt & 
				 0.794 $\pm$ 0.01 & 
				 \textbf{0.808} $\pm$ 0.00 & 
				 0.774 $\pm$ 0.00 & 
				 0.774 $\pm$ 0.00 & 
				 0.748 $\pm$ 0.00 & 
				 0.680 $\pm$ 0.00  
			 \\ 

			 DataOOB & 
				 0.794 $\pm$ 0.01 & 
				 0.788 $\pm$ 0.02 & 
				 0.782 $\pm$ 0.02 & 
				 0.782 $\pm$ 0.02 & 
				 0.789 $\pm$ 0.01 & 
				 0.780 $\pm$ 0.01  
			 \\ 

			 Banzhaf & 
				 0.772 $\pm$ 0.03 & 
				 0.758 $\pm$ 0.03 & 
				 0.748 $\pm$ 0.03 & 
				 0.739 $\pm$ 0.02 & 
				 0.721 $\pm$ 0.04 & 
				 0.678 $\pm$ 0.04  
			 \\ 

			 Random & 
				 0.782 $\pm$ 0.02 & 
				 0.772 $\pm$ 0.02 & 
				 0.765 $\pm$ 0.02 & 
				 0.751 $\pm$ 0.03 & 
				 0.729 $\pm$ 0.03 & 
				 0.675 $\pm$ 0.05  
			 \\ 

		\hline
		\multicolumn{7}{c}{\textit{adult}}\\
		\hline
			 CDVM5k & 
				 \textbf{0.735} $\pm$ 0.01 & 
				 \textbf{0.714} $\pm$ 0.01 & 
				 \textbf{0.714} $\pm$ 0.02 & 
				 \textbf{0.708} $\pm$ 0.01 & 
				 0.703 $\pm$ 0.02 & 
				 \textbf{0.702} $\pm$ 0.01  
			 \\ 

			 CDVM10k & 
				 0.726 $\pm$ 0.02 & 
				 0.709 $\pm$ 0.01 & 
				 0.707 $\pm$ 0.01 & 
				 0.706 $\pm$ 0.01 & 
				 \textbf{0.705} $\pm$ 0.02 & 
				 0.694 $\pm$ 0.01  
			 \\ 

			 InfOpt & 
				 0.717 $\pm$ 0.01 & 
				 0.664 $\pm$ 0.00 & 
				 0.652 $\pm$ 0.00 & 
				 0.638 $\pm$ 0.00 & 
				 0.648 $\pm$ 0.00 & 
				 0.678 $\pm$ 0.00  
			 \\ 

			 DataOOB & 
				 0.715 $\pm$ 0.01 & 
				 0.704 $\pm$ 0.01 & 
				 0.704 $\pm$ 0.01 & 
				 0.698 $\pm$ 0.01 & 
				 0.689 $\pm$ 0.01 & 
				 0.687 $\pm$ 0.01  
			 \\ 

			 Banzhaf & 
				 0.706 $\pm$ 0.02 & 
				 0.678 $\pm$ 0.02 & 
				 0.659 $\pm$ 0.02 & 
				 0.647 $\pm$ 0.03 & 
				 0.628 $\pm$ 0.03 & 
				 0.591 $\pm$ 0.04  
			 \\ 

			 Random & 
				 0.734 $\pm$ 0.02 & 
				 0.698 $\pm$ 0.02 & 
				 0.693 $\pm$ 0.02 & 
				 0.685 $\pm$ 0.02 & 
				 0.673 $\pm$ 0.02 & 
				 0.655 $\pm$ 0.03  
			 \\ 

		\hline
		\multicolumn{7}{c}{\textit{bbc}}\\
		\hline
			 CDVM5k & 
				 0.950 $\pm$ 0.01 & 
				 0.946 $\pm$ 0.01 & 
				 \textbf{0.949} $\pm$ 0.00 & 
				 0.946 $\pm$ 0.01 & 
				 0.941 $\pm$ 0.01 & 
				 0.933 $\pm$ 0.01  
			 \\ 

			 CDVM10k & 
				 0.946 $\pm$ 0.01 & 
				 0.947 $\pm$ 0.01 & 
				 0.947 $\pm$ 0.01 & 
				 \textbf{0.947} $\pm$ 0.01 & 
				 \textbf{0.943} $\pm$ 0.01 & 
				 \textbf{0.934} $\pm$ 0.01  
			 \\ 

			 InfOpt & 
				 \textbf{0.953} $\pm$ 0.01 & 
				 \textbf{0.950} $\pm$ 0.00 & 
				 0.944 $\pm$ 0.00 & 
				 0.934 $\pm$ 0.00 & 
				 0.938 $\pm$ 0.00 & 
				 0.836 $\pm$ 0.00  
			 \\ 

			 DataOOB & 
				 0.944 $\pm$ 0.00 & 
				 0.945 $\pm$ 0.00 & 
				 0.940 $\pm$ 0.00 & 
				 0.938 $\pm$ 0.00 & 
				 0.921 $\pm$ 0.01 & 
				 0.912 $\pm$ 0.01  
			 \\ 

			 Banzhaf & 
				 0.948 $\pm$ 0.01 & 
				 0.945 $\pm$ 0.01 & 
				 0.944 $\pm$ 0.01 & 
				 0.936 $\pm$ 0.01 & 
				 0.905 $\pm$ 0.05 & 
				 0.801 $\pm$ 0.16  
			 \\ 

			 Random & 
				 0.943 $\pm$ 0.01 & 
				 0.940 $\pm$ 0.01 & 
				 0.934 $\pm$ 0.01 & 
				 0.926 $\pm$ 0.02 & 
				 0.901 $\pm$ 0.05 & 
				 0.849 $\pm$ 0.06  
			 \\ 

		\hline

    \end{tabular}
    \caption{Accuracy on 30\%, 25\%, 20\%, 15\%, 10\%, and 5\% of training data for six datasets in the OpenDataVal benchmark \cite{Jiang2023}. Out of 36 configurations, \cdvm achieved state-of-the-art performance in 28 setups. The error margins represent standard deviations based on 25 experiments.}
    \label{tab:results}
\end{table*}